\documentclass{article}
\usepackage{PRIMEarxiv}
\RequirePackage{xcolor}
\RequirePackage[english]{babel}
\RequirePackage[utf8]{inputenc}
\RequirePackage{amsthm,amsmath,amsfonts,amssymb}
\RequirePackage[numbers, sort]{natbib}
\RequirePackage[colorlinks,citecolor=blue,urlcolor=blue]{hyperref}
\RequirePackage{graphicx}
\RequirePackage{etoolbox}

\RequirePackage{mathtools}
\RequirePackage[normalem]{ulem}
\RequirePackage[bb=dsserif]{mathalpha}

\input{mysymbol.sty}
\def\startlocaldefs{\makeatletter}
\def\endlocaldefs{\makeatother}

\startlocaldefs
\theoremstyle{plain}

\newtheorem{theorem}{Theorem}[section]
\newtheorem{corollary}{Corollary}[section]
\newtheorem{lemma}[theorem]{Lemma}
\newtheorem{proposition}[theorem]{Proposition}
\newenvironment{restatedproposition}[1]
{%
  \par\smallskip
  \noindent\textbf{Proposition~\ref{#1}.}
}
{%
  \par
}
\theoremstyle{definition}
\newtheorem{definition}[theorem]{Definition}
\newtheorem*{example}{Example}

\newtheorem{remark}{Remark}[section]
\newtheorem{assumption}{\bf Assumption}
\theoremstyle{remark}

\endlocaldefs

\title{Weighted Random Dot Product Graphs
\thanks{\textit{\underline{Citation}}: 
\textbf{Bernardo Marenco, Paola Bermolen, Marcelo Fiori, Federico Larroca, Gonzalo Mateos. ``Weighted Random
Dot Product Graphs.'' Electronic Journal of Statistics, 20(1) 2456-2499 2026. \href{https://doi.org/10.1214/26-EJS2538}{DOI:10.1214/26-EJS2538.} }} 
}

\author{
  Bernardo Marenco, Paola Bermolen, Marcelo Fiori, Federico Larroca \\
  Facultad de Ingenier\'ia \\
  Universidad de la Rep\'ublica \\
  Uruguay\\
  \texttt{\{bmarenco, paola, mfiori, flarroca\}@fing.edu.uy} \\
   \And
  Gonzalo Mateos \\
  Dept. of Electrical and Computer Engineering \\
  University of Rochester \\
  Rochester, NY, USA\\
  \texttt{gmateosb@ur.rochester.edu} 
}

\begin{document}
\maketitle

\begin{abstract}
Modeling of intricate relational patterns has become a cornerstone of contemporary statistical research and related data science fields. Networks, represented as graphs, offer a natural framework for this analysis. This paper extends the Random Dot Product Graph (RDPG) model to accommodate weighted graphs, markedly broadening the model's scope to scenarios where edges exhibit heterogeneous weight distributions. We propose a nonparametric weighted (W)RDPG model that assigns a sequence of latent positions to each node. Inner products of these nodal vectors specify the moments of their incident edge weights' distribution via moment-generating functions. In this way, and unlike prior art, the WRDPG can discriminate between weight distributions that share the same mean but differ in other higher-order moments. We derive statistical guarantees for an estimator of the nodal's latent positions adapted from the workhorse adjacency spectral embedding, establishing its consistency and asymptotic normality. We also contribute a generative framework that enables sampling of graphs that adhere to a (prescribed or data-fitted) WRDPG, facilitating, e.g., the analysis and testing of observed graph metrics using judicious reference distributions. The paper is organized to formalize the model's definition, the estimation (or nodal embedding) process and its guarantees, as well as the methodologies for generating weighted graphs, all complemented by illustrative and reproducible examples showcasing the WRDPG's effectiveness in various network analytic applications.

\end{abstract}

\keywords{Random weighted graph \and Latent positions model \and Random dot product graph \and Statistical network analysis}

\section{Introduction}\label{sec:intro}

The study of network data has become a central theme in modern statistical research due to the growing demands to understand complex relational structures found in social, biological, and technological systems; see e.g.,~\cite{kolaczyk_2009}. Networks, represented as graphs $G = (V, E)$ with vertex set $V$ and edge set $E$, offer a natural framework to describe pairwise interactions among entities. Statistical approaches to the analysis of network data not only address issues of modeling, inference, and prediction, but also provide insights into the underlying mechanisms governing connectivity and community structure in complex systems.

A powerful approach within this domain is the use of latent variable models~\cite{hoff2002latent}, where each vertex is endowed with an unobserved feature vector that encapsulates its propensity to establish relational ties. This general class of models offers a probabilistic lens through which the observed network structure can be understood in terms of latent geometric or probabilistic constructs, e.g.,~\cite[Ch. 2.2]{kolaczyk_2017} and~\cite[Ch. 10.7]{Izenman_2023}. Among these models, the Random Dot Product Graph (RDPG) stands out for its interpretability and theoretical tractability, providing a bridge between latent geometry and observed connectivity patterns.

In the standard RDPG model, we consider a simple undirected and unweighted random graph $G = (V, E)$ with $ N $ vertices, where each vertex $ i \in V $ is associated with a latent position $ \bbx_i \in \reals^d $. Typically, the dimensionality of the embedding space $d\ll N$. The RDPG model posits that for any pair of vertices $ i $ and $ j $, the presence of an edge is determined by a Bernoulli random variable $ A_{ij}\in\{0,1\} $ with success probability given by the dot product of the latent positions, that is,
\begin{align*}
A_{ij} \sim \text{Bernoulli}(\bbx_i^\top \bbx_j), \quad \text{for } i \neq j,
\end{align*}
where the $A_{ij}$'s are mutually independent. Since self loops are excluded, $A_{ii}\equiv 0$ for all $i\in V$. Here, $ \bbA \in\{0,1\}^{N\times N}$ is the graph's symmetric adjacency matrix, and the inner product $ \bbx_i^\top \bbx_j $ must lie in the interval $[0,1]$ to yield as a valid probability. RDPGs can also incorporate random latent positions by defining a distribution $F$ supported on $\ccalX \subset \reals^d$, such that $\bbx^\top\bby \in [0,1]$ for all $\bbx,\bby \in \ccalX$. 

One can verify that Erd\"os-R\'enyi (ER) graphs, stochastic block models (SBMs) with a positive semidefinite (PSD) probability matrix,  or other more sophisticated random graph models are particular cases of an RDPG~\cite{priebe2018survey}. The resulting expressiveness has been a main driver behind the popularity of the model. Moreover, the RDPG definition provides a natural geometric interpretation of connectivity: angles between latent position vectors indicate affinity between vertices, and their magnitudes indicate how well connected they are. For $d\leq 3$, visual inspection of the nodes' vector representations can reveal community structure. For higher dimensions or more complex scenarios, angle-based clustering of nodal embeddings can also be used~\cite{lyzinski2017community,scheinerman2010rdpg}. 

The associated RDPG inference problem of estimating latent positions from graph observations enjoys strong asymptotic properties, facilitating statistical inference tasks by bringing to bear tools of geometrical data analysis in latent space~\cite{priebe2018survey}. Let us briefly describe the inference method in such a setup. Consider stacking all the nodes' latent position vectors in the rows of matrix $\bbX=[\bbx_1,\ldots,\bbx_N]^\top\in\reals^{N\times d}$. Given an observed graph $\bbA$ and a prescribed embedding dimension $d$ (typically obtained using an elbow rule on $\bbA$'s eigenvalue scree plot, using e.g.~\cite{zhu2006automatic}), the goal is to estimate $\bbX$. For the RDPG model, the edge-formation probabilities are given by the entries \( P_{ij} = \bbx_i^\top \bbx_j \) of the PSD matrix \( \bbP = \bbX\bbX^\top \). Latent positions are uniquely defined only up to an orthogonal transformation, since for any orthogonal matrix \( \bbQ \in O(d) \), the rotated positions \( \bbY = \bbX\bbQ \) satisfy \( \bbP = \bbY\bbY^\top \) as well.
The Adjacency Spectral Embedding (ASE) estimator solves the following least-squares approximation~\cite{scheinerman2010rdpg} to obtain
	\begin{gather}\label{eq:ase_mask}
		\hbX\in\argmin_{\bbX\in \reals^{N\times d}}\Fro{\bbA-\bbX\bbX^\top}^2.
	\end{gather}
In other words, $\hbP=\hbX\hbX^\top$ is the best rank-$d$ PSD approximation to the adjacency matrix $\bbA$, in the Frobenius-norm sense. It is straighforward to see that $\hbX = \hbU\hbD^{1/2}$ verifies \eqref{eq:ase_mask}, where $\bbA=\bbU_\bbA\bbD_\bbA\bbU_\bbA^\top$ is the eigendecomposition of $\bbA$, $\hbD\in\reals^{d\times d}$ is a diagonal matrix with the $d$ largest-magnitude eigenvalues of $\bbA$, and $\hbU\in\reals^{N\times d}$ are the associated eigenvectors from $\bbU_\bbA$.\footnote{Under mild assumptions on the latent positions distribution $F$ it is possible to show that the top $d$ eigenvalues of $\bbA$ are nonnegative with probability tending to 1 as $N\to \infty$, so $\hbX$ is well defined, see \cite{priebe2018survey} for details.} Consistency and asymptotic ($N\to \infty$) normality results for ASE are available; see e.g.,~\cite{priebe2018survey}.

In this work we extend the RDPG model to account for \emph{weighted} graphs, endowed with a weight map $ w: E \mapsto \mathbb{R}_+ $ that assigns a nonnegative value to each edge. The adjacency matrix entries for the weighted graph are $ W_{ij} = W_{ji} = w(i,j) $ for all $ (i,j) \in E $, while the absence of an edge is represented as $ W_{ij} = W_{ji} = 0 $. This formulation naturally encompasses unweighted graphs as a special case, where edge weights are constrained to binary values (i.e., $ w \equiv 1 $).  

Several extensions of the standard (or \textit{vanilla}) RDPG model have been developed to incorporate edge weights. A common approach is to introduce a parametric distribution $F_{\bbtheta}$ with parameter $ \bbtheta \in \reals^L $ to model weighted adjacency entries~\cite{deford2016random,tang2017robust}. For instance, in the case of Poisson-distributed weights, one might set $ L=1 $ and assume $ W_{ij} \sim \text{Poisson}(\theta) $. In this formulation, each node $ i \in V $ is assigned a collection of latent vectors $ \bbx_i[l] \in \reals^{d} $ for $ l = 1, \dots, L $, such that the weight of an edge between nodes $ i $ and $ j $ follows the distribution  
\begin{align*}
W_{ij} \sim F_{(\bbx_i^\top[1] \bbx_j[1], \dots, \bbx_i^\top[L] \bbx_j[L])},
\end{align*}
independently across edges. To model sparse graphs, this distribution may include a point mass at zero.
The classic RDPG model is recovered by setting $ F_{\bbtheta} $ as a Bernoulli distribution with success probability $ \theta \in [0,1]$. Despite its flexibility, this approach has some noteworthy limitations. A key drawback is that all edges must share the same parametric family of weight distributions, differing only in their parameters. While mixture models can introduce some heterogeneity, they still require an explicit assumption about the underlying distribution families. This requirement significantly restricts the model's applicability, particularly if edges follow heterogeneous, unknown, or multimodal weight distributions.  

To address this shortcoming, recent work by Gallagher et al.~\cite{gallagher2023spectral} proposes a nonparametric alternative. Therein each node has a single associated latent position $\bbz_i\in\ccalZ$, which is endowed with a probability distribution $F$.  It is postulated that, given a family $\{H(\bbz_1,\bbz_2): \bbz_1,\bbz_2 \in \ccalZ \}$ of symmetric real-valued distributions, there exists a map $\bbphi:\ccalZ \mapsto \reals^d$ such that if $W_{ij}\sim H(\bbz_i,\bbz_j)$, then $\E{W_{ij}} = \bbphi^\top(\bbz_i) \bbI_{p,q}\bbphi(\bbz_j)$, where $\bbI_{p,q}$ is a diagonal matrix of $p$ ones and $q$ minus ones, such that $p+q=d$. This diagonal matrix facilitates modeling heterophilic (or disassortative) behavior, as in \cite{rubin2022statistical}. Interestingly, one can consistently recover $\bbx_i = \bbphi(\bbz_i)$ via the ASE of $\bbW$, and the estimated $\hbx_i$'s are asymptotically normal. While this method improves flexibility and avoids restrictive parametric assumptions, it can only recover $\bbphi(\bbz_i)$, i.e., the latent positions for the mean adjacency matrix $\E{\bbW}$. Consequently, the model is unable to differentiate between pairs of edges that stem from distinct distributions sharing an identical mean. These challenges highlight the ongoing need for more expressive, discriminative, and robust latent space models for weighted graphs.  

\subsection{Proposed approach and contributions}
Instead, we propose that the sequence of nodal vectors be related to the weight distribution's moment-generating function (MGF). Assume each node has a sequence of latent positions $\{\bbx_i[k]\}_{k \geq 0}$ such that the inner products $\{\bbx_i^\top[k] \bbx_j[k]\}_{k \geq 0}$ form an admissible moment sequence. Our weighted (W)RDPG model specifies the $k$-th order moments of the random entries in the adjacency matrix $\bbW \in \reals^{N\times N}$ are given by $\E{W_{ij}^k}= \bbx_i^\top[k] \bbx_j[k]$, for all $k\geq 0$ (see Section \ref{ssec:model_def}).

The WRDPG model offers several key advantages: (i) it is nonparametric; (ii) it considers higher-order moments beyond the mean; while providing (iii) statistical guarantees for the associated estimator; and (iv) a generative mechanism that can reproduce the structure and the weights' distribution of real networks. Next, we elaborate on these attractive features to better position our technical contributions in context. Unlike~\cite{deford2016random,tang2017robust}, the nonparametric WRDPG graph modeling framework  does not require prior assumptions about a specific weight distribution. Nodal vectors can be used to describe high-order moments of said distribution, thus enhancing the model's representation and discriminative  power. For example,  different from~\cite{gallagher2023spectral} the WRDPG model can distinguish between distributions with the same mean but differing standard deviations; see also Section \ref{sec:discriminative_power} and Figure \ref{fig:embedding_norm_poisson} for an illustrative example. 

\begin{figure}[t]
    \centering
    \includegraphics[width=\textwidth]{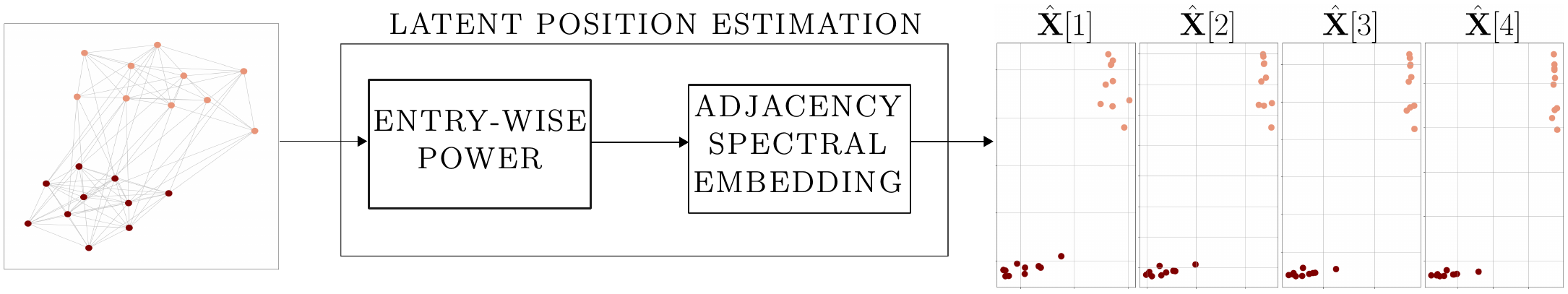}
    \caption{Latent position estimation. Given an adjacency matrix $\bbA$ we compute its $k$-th entry-wise power $\bbA^{(k)}$. The ASE of $\bbA^{(k)}$ yields the estimates $\hbX[k]$; see also Section \ref{ssec:estimation}.}
    \label{fig:wrdpg_estimation_blocks}
\end{figure}

Furthermore, framing our model in terms of (symmetric) adjacency matrices allows us to define latent position estimators based on their spectral decomposition; see Section \ref{ssec:estimation}. Statistical guarantees can be derived using tools to control the matrix spectrum under perturbations. As shown in \cite{marenco2021tsipn}, for each fixed moment index $k$, the latent position matrix $\bbX[k]$ can be consistently recovered (up to an orthogonal transformation) via the ASE of $\bbW^{(k)}$, the matrix of entry-wise $k$-th powers of $\bbW$; see Figure \ref{fig:wrdpg_estimation_blocks} for a schematic diagram of the latent position estimation procedure. In Section \ref{sec:asymptotic_results} we generalize the result in~\cite{marenco2021tsipn} by proving that consistency holds for unbounded sub-Weibull edege weights under a stronger convergence metric that ensures uniform convergence for all $\hbx_i$'s. We also establish that the estimator is asymptotically normal  as $N\to\infty$, a result that is novel in our WRDPG setup.

In Section \ref{sec:generative} we show how one can generate graphs adhering to the proposed WRDPG model, which is non-trivial in a nonparameteric setting and was not considered in~\cite{gallagher2023spectral}. We develop a methodology to sample random weighted graphs whose edge weight distribution is defined by a sequence of moments given by inner products of connected nodes' latent positions. Depending on the characteristics of the weight distribution--whether discrete, continuous, or a mixture--we can solve for the latent position sequence to ensure the inner products match the prescribed moments. For real-valued (continuous) weights, we rely on the maximum entropy principle subject to moment constraints. We introduce a primal-dual approach to find a probability density function that maximizes entropy, offering improved numerical stability relative to previous algorithms~\cite{saad2019pymaxent}. Results for mixture distributions allow us to generate graphs that simultaneously replicate the sparsity pattern of the network and the edge weights. In Figure \ref{fig:wrdpg_generation_blocks} we present a schematic diagram for the graph generation pipeline.

\begin{figure}[t]
    \centering
    \includegraphics[width=\textwidth]{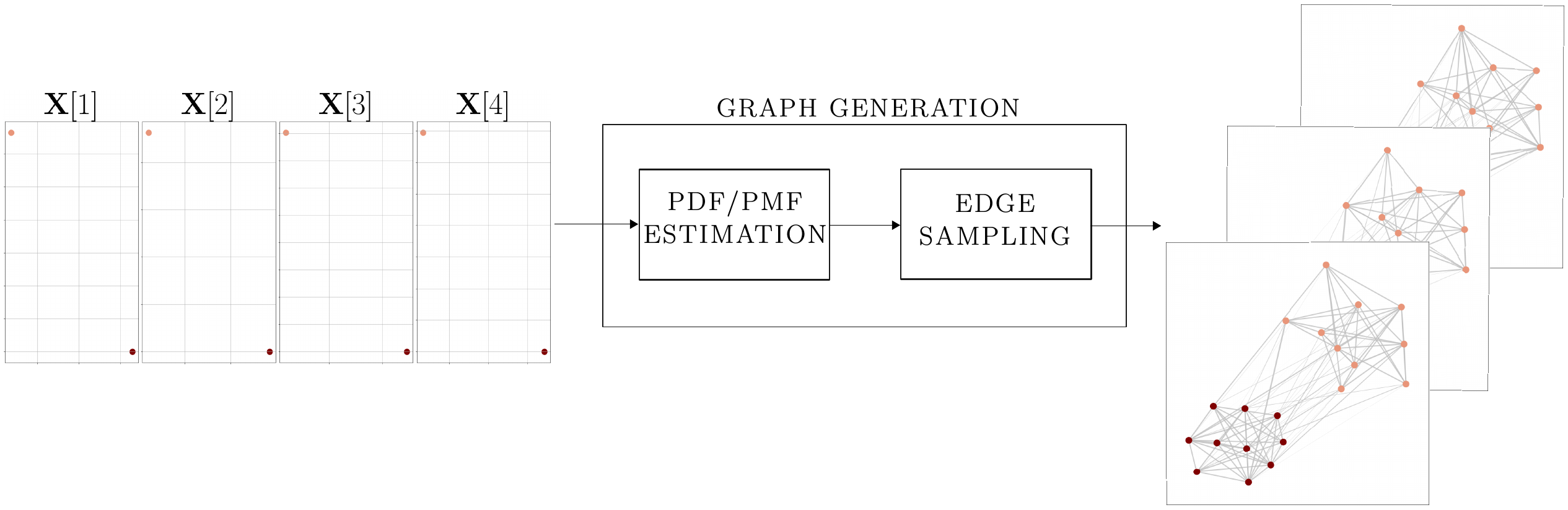}
    \caption{Graph generation. Given the latent positions of each vertex $\{\bbX[k]\}_{k \geq 0}$, we estimate a weight distribution whose sequence of moments is given by the corresponding dot products. Edge weights are then sampled from this estimated distribution; see also Section \ref{sec:generative}.}
    \label{fig:wrdpg_generation_blocks}
\end{figure}

Moreover, given an observed weighted network (instead of the ground-truth latent position sequence as in Figure \ref{fig:wrdpg_generation_blocks}), we show how we can use this generative procedure to sample graphs that are similar to the original one, in a well-defined statistical sense. This can be useful for several statistical inference tasks involving network data. If, for instance, one would like to assess the significance of some observed graph characteristic, this procedure allows to generate several `comparable' graphs and construct judicious reference distributions~ \cite[Section 6.2]{kolaczyk_2009}. In a related use case, one could perform hypothesis testing to determine if a given graph adheres to the WRDPG model.

The rest of the paper is organized as follows. In Section \ref{sec:model} we formally define the WRDPG model, describe the embedding method, and present several illustrative examples. The estimation problem is addressed in Section \ref{sec:asymptotic_results}, and the proofs of the asymptotic statistical guarantees are presented in detail. Finally, in Section \ref{sec:generative}, we present our methodology for generating WRDPG graphs.  We include several reproducible examples to demonstrate the ability of the generative framework to produce weighted graphs with desired characteristics. 

\subsection*{Notational conventions}
Throughout the paper we use the following notation. Real numbers are denoted by plain symbols (either upper or lowercase), such as $i,j,N\in \reals$. Vectors in a Euclidean space $\reals^d$ are denoted in bold, lowercase, e.g., $\bbx \in \reals^{d}$, and are assumed to be column vectors. We reserve bold uppercase letters for matrices, such as $\bbA \in \reals^{n\times m}$, with $\bbI$ denoting the identity matrix. The symbol $\bbA^{(k)}$ stands for the $k$-fold Hadamard (i.e., entrywise) product of matrix $\bbA$ with itself, that is:
\begin{align*}
   \bbA^{(k)} := \underbrace{\bbA \circ \bbA \circ \dots \circ \bbA}_{k \text{ times}} .
\end{align*}
The group of orthogonal matrices of dimension $d$ is denoted by $O(d)$, i.e., $\bbQ \in O(d)$ iff $\bbQ \in \reals^{d\times d}$ and $\bbQ\bbQ^\top = \bbI$, where $(\cdot)^\top$ stands for transposition.

The dot product between $\bbx, \bby \in \reals^d$ is denoted by $\bbx^\top \bby$, whereas $\twonorm{\bbx}$ and $\infnorm{\bbx}$ denote the Euclidean (i.e., $L_2$) and maximum (i.e., $L_\infty$) norm of vector $\bbx$. When dealing with matrices, $\Fro{\bbA}$, $\maxnorm{\bbA}$, $\spectral{\bbA}$ and $\twoToInf{\bbA}$ denote the Frobenius, maximum (i.e., largest entry in magnitude), spectral (i.e., largest singular value) and $2\to \infty$ norm of matrix $\bbA\in \reals^{n\times m}$, respectively. The latter is defined as
\begin{displaymath}
    \twoToInf{\bbA} := \sup_{\twonorm{\bbx}=1} \infnorm{\bbA\bbx},
\end{displaymath}
see \cite{cape2019twoToInfinity}. The $2\to\infty$ norm is thus the operator norm of $\bbA$ acting as $\bbA(\cdot):\reals^n \mapsto \reals^m$, when $\reals^n$ is endowed with the Euclidean norm and $\reals^m$ with the $\infty$ norm.

We repeatedly make use of the following relations between the $2\to\infty$, spectral, and Frobenius norms of any matrix $\bbA \in \reals^{n\times m}$:
\begin{align*}
        \twoToInf{\bbA} \leq &\, \spectral{\bbA} \leq \Fro{\bbA}, \\
         \spectral{\bbA} \leq &\min \left\lbrace \sqrt{n}\twoToInf{\bbA}\,, \sqrt{m} \twoToInf{\bbA^\top} \right\rbrace .
\end{align*}
The $2\to\infty$ norm is not submultiplicative; however, the following relations hold for any matrices $\bbA,\bbB, \bbC$ of suitable dimensions:
\begin{align*}
    \twoToInf{\bbA\bbB} \leq& \twoToInf{\bbA}\spectral{\bbB} \leq \spectral{\bbA} \spectral{\bbB},\\
    \twoToInf{\bbC\bbA} \leq& \infnorm{\bbC}\twoToInf{\bbA}\leq \sqrt{n} \spectral{\bbC}\spectral{\bbA},
\end{align*}
where $\infnorm{\bbA} := \max_i \sum_j |A_{ij}|$ is the maximum absolute row sum of $\bbA$.

For a function $f:\naturals \mapsto \reals^+$, we say that a random variable $Y\in\reals$ is $\OP{f(N)}$ if for each $\alpha > 0$ there exists $N_0\in \naturals$ and $C>0$ such that for all $N\geq N_0$ it holds that $\P{|Y|<Cf(N)} \geq 1-N^{-\alpha}$. In that case we write $Y=\OP{f(N)}$. Furthermore, we say that $Y$ is $\EquivP{f(N)}$ if for each $\alpha > 0$ there exists $N_0\in \naturals$ and $c,C>0$ such that for all $N\geq N_0$ it holds that $\P{cf(N) <Y<Cf(N)} \geq 1-N^{-\alpha}$. In that case we write $Y=\EquivP{f(N)}$.

\section{Weighted RDPG model}\label{sec:model}

Here we define the WRDPG model and the ASE-based latent position estimator. Illustrative examples are presented to ground these concepts. Furthermore, we show the discriminative power for community detection inherited from considering higher-order moments beyond the mean. Finally, we discuss the impact of the number of nodes on the accuracy of the moment sequence reconstruction.

\subsection{Model specification}\label{ssec:model_def}
In this section, we formally define our WRDPG model. We follow the rationale of the vanilla RDPG and define the model in terms of random latent position sequences per node. As previewed in Section \ref{sec:intro}, the inner product of the latent positions will be related to the MGF of the edge weights' distribution, thus requiring some preliminary definitions.

\begin{definition}[Admissible moment sequence]\label{def:moment_sequence}
A sequence $\{m[k]\}_{k \geq 0}$ of real numbers is an admissible moment sequence if $m[0]=1$ and for all $p \geq 0$ the matrix
\begin{displaymath}
    \bbM = \left(
        \begin{array}{ccccc}
            m[0] & m[1] & m[2] & \dots & m[p] \\
            m[1] & m[2] & m[3] & \dots & m[p+1]\\
            m[2] & m[3] & m[4] & \dots & m[p+2]\\
            \vdots & \vdots & \vdots & \ddots & \vdots \\
            m[p] & m[p+1] & m[p+2] & \dots & m[2p]
        \end{array}
    \right) \in \reals^{(p+1)\times (p+1)}
\end{displaymath}
whose entries are $M_{ij} = m[i+j-2]$ is PSD.
\label{def:moments_sequence}
\end{definition}

\begin{remark}
Definition \ref{def:moments_sequence} is enough to guarantee that there exists a probability measure $\mu$ in $\reals$ such that $m[k]$ is the $k$-th moment of $\mu$; see e.g.,~\cite[Ch. 3]{widder1941laplace}.
\end{remark}

\begin{definition}[Weighted inner-product distribution]\label{def:inner_product_dist}
Let $F$ be a probability distribution with support $\text{supp} F = \mathcal{X} \subset (\reals^d)^\infty$. We say that $F$ is a weighted inner-product distribution if for all $\{\bbx[k]\}_{k\geq 0}, \{\bby[k]\}_{k \geq 0} \in \mathcal{X}$, the sequence $\{\bbx^\top[k] \bby[k]\}_{k \geq 0}$ is an admisible moment sequence.
\end{definition}

\begin{definition}[WRDPG]\label{def:wrdpg_model}
Let $F$ be a weighted inner-product distribution and let $\{\bbx_1[k]\}_{k \geq 0}$, $\{\bbx_2[k]\}_{k\geq 0},$ $\dots,\{\bbx_N[k]\}_{k\geq 0}$ be i.i.d. with distribution $F$. Given the sequence $\bbX_k:=\{\bbX[k]\}_k$, with $\bbX[k]=[\bbx_1[k],\ldots, \bbx_N[k]]^\top\in \reals^{N\times d}$, the weighted (W)RDPG model prescribes the $k$-th order moments of a random adjacency matrix $\bbW \in \reals^{N\times N}$ are given by $\E{W_{ij}^k}= \bbx_i^\top[k] \bbx_j[k]$, for all $k\geq 0$. Formally, the WRDPG model specifies the MGF of $\bbW \in \reals^{N\times N}$ as
\begin{equation}\label{eq:mgf}
	\E{e^{tW_{ij}}|\bbX_k} = \sum_{k=0}^\infty \frac{t^k\E{W_{ij}^k}}{k!} = \sum_{k=0}^\infty \frac{t^k\bbx_i^\top[k]\bbx_j[k]}{k!}
\end{equation}
and the entries $W_{ij}$ are independent, i.e., edge independent. In such a case, we write $(\bbW,\bbX_k) \sim \mathrm{WRDPG}(F)$.
\end{definition}

\begin{remark}[Nonidentifiability of latent positions] As is the case for the vanilla RDPG model, latent positions are invariant to orthogonal transformations, since for any $\bbQ \in O(d)$, $\bbX[k]\bbX^\top[k]=\bbX[k]\bbQ \left(\bbX[k]\bbQ\right)^\top$. This implies that given a latent position sequence $\bbX_k$, for each index $k$ one may choose a matrix $\bbQ_k\in O(d)$ (which may vary with $k$) and construct a sequence $\bbY_k:=\{\bbX[k] \bbQ_k \}_k$, which will result in the same distribution of graphs as that given by $\bbX_k$.
\end{remark}

\begin{remark}[Model extensions]
    For simplicity of exposition, the WRDPG model is defined assuming a common embedding dimension $d$ across all moment indices $k$. More generally, the framework can be extended to allow the latent positions $\bbX[k]$ to lie in moment-specific dimensions $d_k$, leading to representations of the form $\E{W_{ij}^k} = \bbx_i^\top[k]\bbx_j[k]$, with $\bbx_i[k] \in \mathbb{R}^{d_k}$. Such dimensions could be selected in a data-driven manner in practical applications, for instance by applying an elbow rule–based procedure to the eigenvalue scree plot of the corresponding empirical $k$-th order moment matrix.

    In addition, the WRDPG model can be generalized to allow for indefinite inner products by replacing $\bbx_i^\top[k]\bbx_j[k]$ with $\bbx_i^\top[k]\bbI_k^{p,q}\bbx_j[k]$, where
    \begin{displaymath}
        \bbI_k^{p,q} := \mathrm{diag}(\underbrace{1,\ldots,1}_{p_k},\underbrace{-1,\ldots,-1}_{q_k}), \qquad p_k + q_k = d_k.
    \end{displaymath}
    This extension accommodates heterophilous graph structures and aligns the WRDPG framework with the generalized RDPG model of~\cite{rubin2022statistical}. Since our methodology builds on the spectral techniques of~\cite{gallagher2023spectral}, which explicitly account for such indefinite inner product structures, analogous asymptotic guarantees to those developed in Section~\ref{sec:asymptotic_results} are expected to hold in this generalized setting.
\end{remark}

\subsection{Estimation of latent positions}\label{ssec:estimation}

The vectors \( \bbx_i[k] \) can be estimated via an ASE of the matrix \( \bbW^{(k)} \), the entrywise \( k \)-th power of an observed symmetric weighted adjacency matrix \( \bbW \). Indeed, for each index \( k \), we stack the latent positions of all nodes into the matrix \( \bbX[k] = [\bbx_1[k], \ldots, \bbx_N[k]]^\top \in \reals^{N \times d} \). If one had access to the moment matrix \( \bbM_k := \bbX[k]\bbX^\top[k] \), it would be straightforward to recover \( \bbX[k] \) (up to an orthogonal transformation) from the eigendecomposition of \( \bbM_k \). However, since \( \bbM_k \) is typically unobserved, we instead rely on the observed weighted adjacency matrix \( \bbW \). Because each entry of \( \bbW^{(k)} \) has expectation equal to the corresponding entry of \( \bbM_k \), we approximate \( \bbX[k] \) by solving [cf. \eqref{eq:ase_mask}]
\begin{align*}
	\hbX[k] \in \argmin_{\bbX \in \reals^{N \times d}} \Fro{\bbW^{(k)} - \bbX\bbX^\top}^2,
\end{align*}
for all \( k \geq 0 \). A solution is readily obtained by setting \( \hbX[k] = \hbU_k \hbD_k^{1/2} \), where \( \bbW^{(k)} = \bbU_{W_k} \bbD_{W_k} \bbU_{W_k}^\top \) is the eigendecomposition of \( \bbW^{(k)} \), \( \hbD_k \in \reals^{d \times d} \) is the diagonal matrix of the \( d \) largest-magnitude eigenvalues, and the columns of \( \hbU_k \in \reals^{N \times d} \) contain the corresponding eigenvectors. In Appendix~\ref{app:eigenvalues}, we show that, under mild assumptions on the weighted inner-product distribution \( F \) (Assumptions~\ref{assumption:full_Rank} and~\ref{assumption:bounded_rv} in Section~\ref{sec:asymptotic_results}), the top \( d \) eigenvalues of \( \bbW^{(k)} \) are nonnegative, ensuring that \( \hbX[k] \) is well-defined. We refer to this estimator as the ASE of \( \bbX[k] \).

\subsection{Examples}\label{sec:examples}

Next, we derive expressions for the latent position vectors in some simple, yet classical, models. We also simulate such WRDPG instances and show that latent position estimation via the ASE yields satisfactory results.  We will derive statistical guarantees for ASE in Section \ref{sec:asymptotic_results}, justifying our empirical findings.

\subsubsection{Erdös-Rényi graph with Gaussian weights}

Consider a graph with $N=1000$ nodes. We first sample the presence or absence of each edge independently as $\mathrm{Bernoulli}(p)$. Then, for each edge we independently sample its weight from a $\mathcal{N}(\mu,\sigma^2)$ distribution. This means that $W_{ij}$ is either 0 with probability $1-p$ or follows a $\mathcal{N}(\mu,\sigma^2)$ distribution with probability $p$. One can show that the latent position sequence for every node is $\bbx[0]=1$ and $\bbx[k]=\sqrt{pm_{\ccalN}[k]}$ for $k\geq 1$, where $m_{\ccalN}[k]$ is the $k$-th moment of the normal distribution with parameters $\mu$ and $\sigma^2$.

Figure \ref{fig:ER_normal_example} shows a histogram of the recovered latent positions up to order $k=6$, with $p=0.5$, $\mu=1$ and $\sigma=0.1$. The red dashed line indicates the true latent positions in each case. Apparently, for each $k$ the estimated positions follow a normal distribution centered around the ground-truth value $\bbx[k]$. This observation is supported by plotting the pdf of the limiting Gaussian distribution -- i.e., the asymptotic distribution of the estimated embeddings as $N\to\infty$, according to Theorem \ref{theorem:tcl}, which is proven in Section~\ref{sec:asymptotic_results}; see also Corollary \ref{corollary:tcl_sbm}.

\begin{figure}[t]
    \centering
    \includegraphics[width=0.8\textwidth]{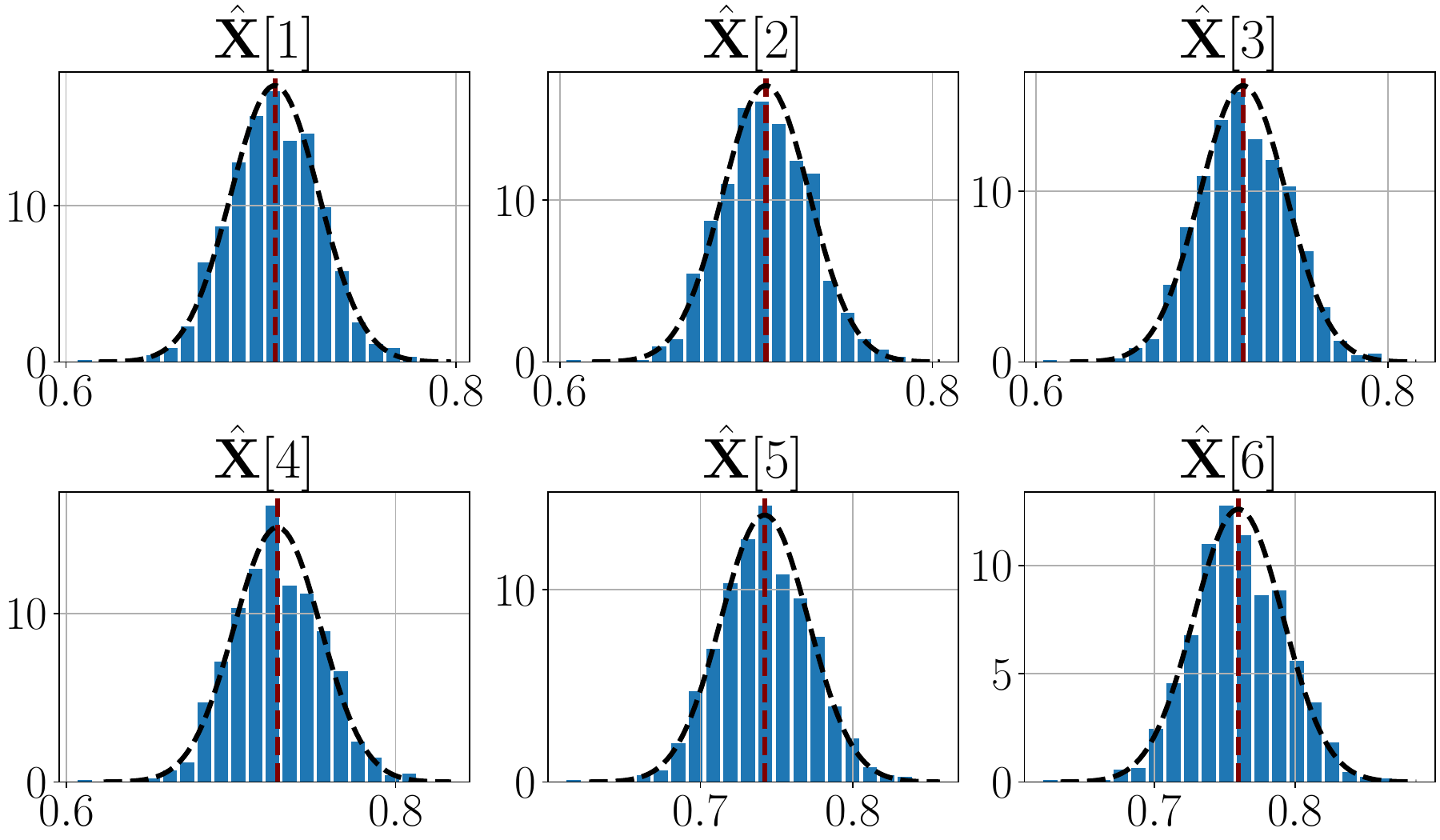}
    \caption{True (dashed vertical line) and estimated (histograms) latent positions for an Erdös-Rényi model with $\mathcal{N}(1,0.01)$ weights. Pdf for limiting Gaussians, as given by Corollary \ref{corollary:tcl_sbm}, are plotted with dashed lines in each panel. The empirical distributions closely follow their limiting counterparts derived from the asymptotic normality results established in Section~\ref{sec:asymptotic_results}.}
    \label{fig:ER_normal_example}
\end{figure}

\subsubsection{Two-block SBM with arbitrary weights' distribution}\label{ex:wsbm}
In this case, the setup is similar as before; only now is the presence or absence of an edge given by a 2-block SBM. That is, each node belongs to exactly one of two communities, and an edge is formed between two nodes with probability $p_1$ if both belong to community 1, with probability $p_2$ if both belong to community 2, and with probability $q$ if they belong to different communities. So, conditioned on individual node assignments to communities, the presence or absence of edges is given by the block probability matrix:
\begin{equation}\label{eq:2_block_sbm}
\bbB = \left(
\begin{array}{c c}
p_1 & q\\
q & p_2
\end{array}
\right).
\end{equation}
After sampling edges, all weights are independently sampled from the same arbitrary edge-weight distribution.

Let $\bbx_{C_1},\bbx_{C_2} \in (\reals^2)^\infty$ denote the latent position sequence for each community. To compute the analytical latent positions for this model, we begin by looking at the mean, i.e., $\bbx_{C_1}[1],\bbx_{C_2}[1] \in \reals^2$. If two nodes belong to community 1, then an edge is formed between them with probability $p_1$, and its weight follows the prescribed distribution. This implies that
\begin{displaymath}
\bbx_{C_1}^\top[1] \bbx_{C_1}[1] = p_1m_d[1] \Rightarrow \|\bbx_{C_1}[1]\| = \sqrt{p_1 m_d[1]},
\end{displaymath}
where $m_d[1]$ is the first moment (the mean) of the chosen weights' distribution. Due to the nonidentifiability of latent positions, we can arbitrarily place the latent positions for community 1 along the $x$-axis, and therefore choose:
\begin{equation}\label{eq:sbm_normal_mean_c1}
\bbx_{C_1}[1] =(\sqrt{p_1 m_d[1]},0)^\top.
\end{equation}
Similarly, we conclude that for community 2 we must have $\|\bbx_{C_2}[1]\| = \sqrt{p_2m_d[1]}$. Also, an edge between two nodes in different communities is present with probability $q$ and its weight follows the prescribed distribution, so
$$\bbx_{C_1}^\top[1] \bbx_{C_2}[1] = qm_d[1].$$
From these two constraints, using \eqref{eq:sbm_normal_mean_c1} we conclude that
\begin{displaymath}\label{eq:sbm_normal_mean_c2}
\bbx_{C_2}[1] = \left(q \sqrt{\frac{m_d[1]}{p_1}},\sqrt{m_d[1] \left(p_2-\frac{q^2}{p_1}\right)} \right)^{\top} .
\end{displaymath}

Following the analysis above, one can derive the higher-order terms in the per-community sequences. Indeed, if $m_d[k]$ is the $k$-th moment for the chosen weights' distribution, imposing that the inner products between latent positions equals the moments of inter- and intra-communities connections leads to:
\begin{align*}
\|\bbx_{C_1}[k]\|^2 &= p_1m_d[k]\\
\|\bbx_{C_2}[k]\|^2 &= p_2m_d[k]\\
\bbx_{C_1}^\top[k] \bbx_{C_2}[k] &= qm_d[k] .
\end{align*}
Again, arbitrarily placing the latent positions for community 1 along the $x$-axis we can find the latent positions for each community:
\begin{equation} \label{eq:sbm_normal_mean_latent_sequence}
\begin{split}
\bbx_{C_1}[k] &= \left( \sqrt{p_1m_d[k]}, 0 \right)^\top \\
\bbx_{C_2}[k] &= \left(q \sqrt{\frac{m_d[k]}{p_1}},\sqrt{m_d[k] \left(p_2-\frac{q^2}{p_1}\right)}\, \right)^\top .
\end{split}
\end{equation}

Note that \eqref{eq:sbm_normal_mean_latent_sequence} is valid for $k\geq 1$. For $k=0$ we can arbitrarily set $\bbx_{C_i}[0]=(1,0)^\top$ for $i=1,2$ in order to: \textit{a}) maintain the dimensionality of the embeddings; and \textit{b}) force the $0$-th moment of each edge to be equal to 1 [cf. Definition \ref{def:moment_sequence}].

We simulated the above setup for a network with $N=1000$ nodes ($700$ in community 1 and $300$ in community 2), block probability matrix
\begin{displaymath}
\bbB = \left(
\begin{array}{c c}
0.7 & 0.1\\
0.1 & 0.3
\end{array}
\right),
\end{displaymath}
and weights sampled from a $\ccalN(\mu,\sigma^2)$ distribution, with $\mu=1$ and $\sigma=0.1$. Results for the ASE of $\bbW^{(k)}$ up to order $k=6$ are shown in Figure \ref{fig:SBM_normal_example}. As expected, the estimated embeddings for each community are centered around the analytically derived ones in \eqref{eq:sbm_normal_mean_latent_sequence}, with an ellipse-like outline that suggests an approximately normal distribution. As with the previous ER example, this is corroborated by plotting the 95\% confidence level sets of the limiting Gaussians for each community, as given by the explicit formulae in  Corollary \ref{corollary:tcl_sbm}.

\begin{figure}[t]
    \centering
    \includegraphics[width=0.8\textwidth]{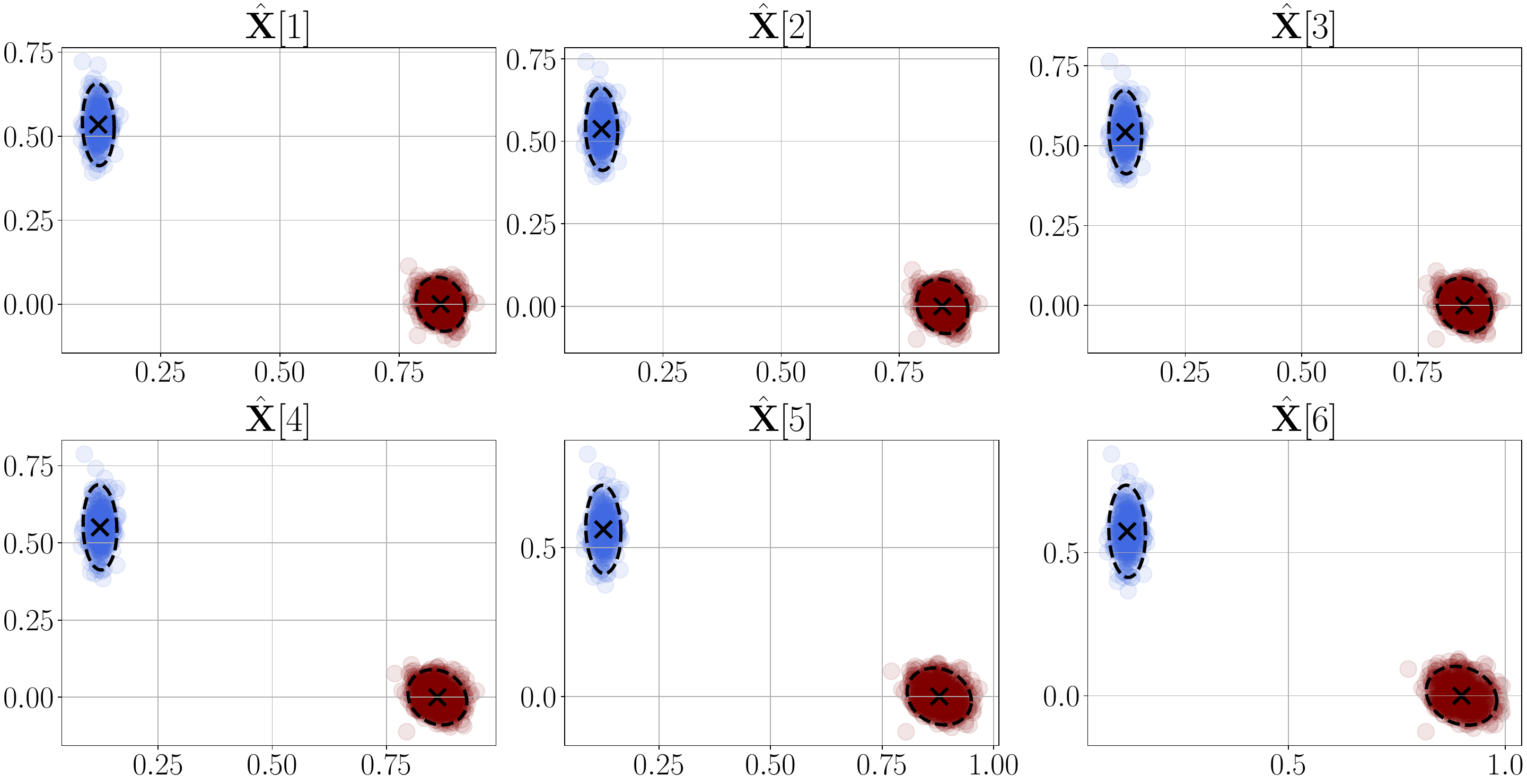}
    \caption{Estimated (blue and red circles) and true latent positions (black crosses) for a two-block SBM with $\mathcal{N}(1,0.01)$ weights. The 95\% confidence level sets for the limiting normal distributions, as given by Corollary \ref{corollary:tcl_sbm}, are shown as dashed lines. Once more, the accuracy of the theoretical predictions is apparent.}
    \label{fig:SBM_normal_example}
\end{figure}

\subsection{Discriminative power of higher-order spectral embeddings}
\label{sec:discriminative_power}
To demonstrate the ability of our model to distinguish between communities, we simulate a two-block weighted SBM consisting of \( N = 2000 \) nodes. Edges are established independently with probability \( p = 0.5 \). Edge weights follow a Gaussian distribution with mean \( \mu = 5 \) and standard deviation \( \sigma = 0.1 \), except for those between the second block of nodes, indexed \( i = 1001, \ldots, 2000 \), where the weights instead follow a Poisson distribution with rate parameter \( \lambda = 5.1 \).

For the same reasons as in the previous example, in this setting the matrix \( \bbX[k] \) has at most two distinct columns for each \( k \). Figure~\ref{fig:embedding_norm_poisson} displays the estimated embeddings \( \hbx_i[k] \) obtained via the ASE of $\bbW^{(k)}$ for \( k = 1, 2, 3 \) with embedding dimension \( d = 2 \), where nodes are color-coded by community membership. The 95\% confidence level sets for the limiting normal distributions, as given by Theorem \ref{theorem:tcl}, are shown in Figure \ref{fig:embedding_norm_poisson} as dashed lines. Notably, for each value of $k$, the simulated points closely follow the normal distribution predicted by the theorem in the large-$N$ limit, providing empirical support for the asymptotic result.

Observe that for \( k = 1 \), the node embeddings are nearly indistinguishable across the two communities. This is expected, as the vectors \( \hbx_i[1] \) cluster around the points \( (\sqrt{\mu p}, 0)^\top  \approx (1.58, 0)^\top \) for the gaussian community, and \( (\sqrt{\mu p}, \sqrt{p(\lambda-\mu)})^\top \approx (1.58, 0,22)^\top \) for the Poisson community, reflecting the almost identical expected weight of edges in both distributions. When \( k = 2 \), the embeddings begin to reveal community structure, as shown in the center panel of Figure~\ref{fig:embedding_norm_poisson}. In this simplified example, closed-form expressions for higher-order embeddings can be readily derived as well. Assuming, without loss of generality, that the embeddings \( \bbx_i[k] \) lie along the $x$-axis for \( i = 1, \ldots, 1000 \), we obtain:
\[
\mathbf{x}_i[k] = 
\begin{cases}
\left(\sqrt{p m_k^N}, 0\right)^\top, & i \leq 1000, \\
\left(\sqrt{p m_k^N}, \sqrt{p(m_k^P - m_k^N)}\right)^\top, & i > 1000,
\end{cases}
\]
where $m_k^N$ and $m_k^P$ are the $k$-th moments of the univariate $\ccalN(\mu,\sigma^2)$ and Poisson distributions, respectively.
For $k=2$, these correspond to the approximate coordinates \( (3.54, 0)^\top \) and \( (3.54, 1.75)^\top \) for the two groups. However, since the confidence sets of the limiting multivariate Gaussians still overlap, the separation between the communities is not yet clearly pronounced. At \( k = 3 \), the confidence sets no longer intersect, and the embeddings exhibit a clear separation between the two blocks, as shown in the right panel of Figure~\ref{fig:embedding_norm_poisson}.

\begin{figure}[t]
	\centering
	\includegraphics[width=\linewidth]{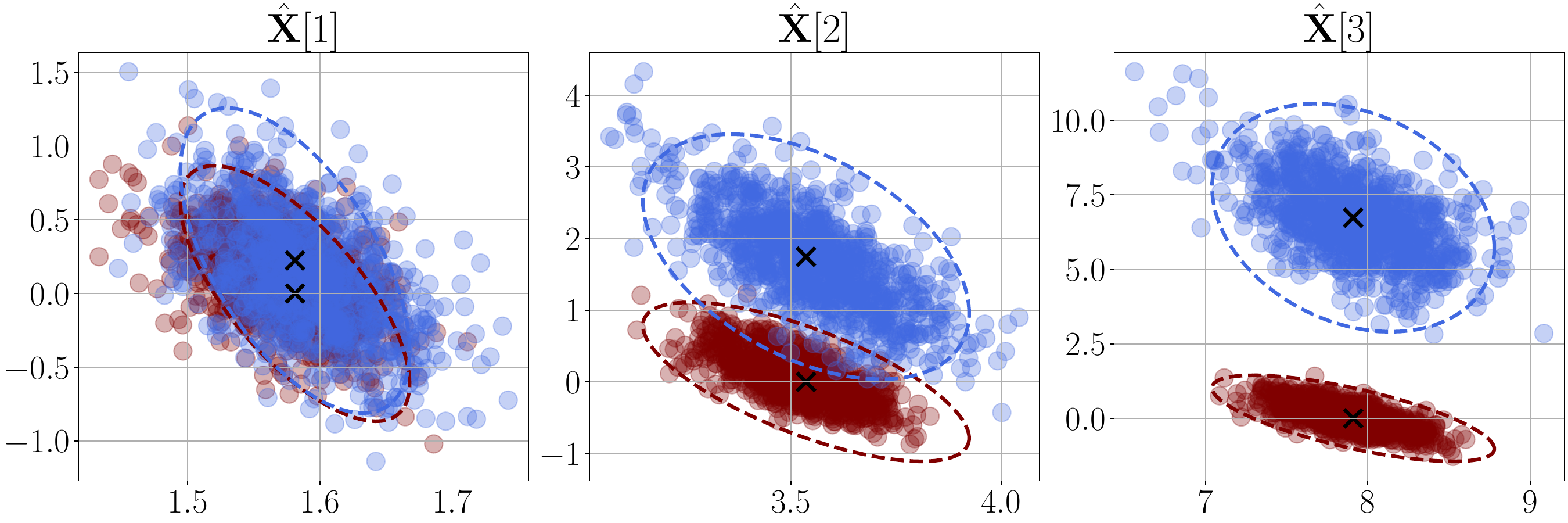}
	\caption{Theoretical latent positions (black crosses) and ASE embeddings of $\bbW^{(k)}$ for Gaussian ($\mu=5$ and $\sigma=0.1$; in red) and Poisson ($\lambda=5.1$; in blue) distributed weights for $d=2$ and $k=1$ (left), $k=2$ (center), and $k=3$ (right). Nodes with different weight distributions are clearly revealed for $k=3$, but they overlap for $k=1$. The 95\% confidence level sets for the limiting normal distributions, as given by Theorem \ref{theorem:tcl}, are shown as dashed lines.
    }
	\vspace{-0.3cm}
	\label{fig:embedding_norm_poisson}
\end{figure}

Recall that the model proposed in~\cite{gallagher2023spectral} is limited to embeddings derived solely from the mean of the weight distribution, specifically, the ASE of the matrix \( \bbW^{(1)} = \bbW \). In the weighted SBM described above, both blocks are constructed to have almost identical expected edge weights. As a consequence, the embeddings \( \hbx_i[1] \) obtained via an ASE of \( \bbW \) are concentrated around the same point in latent space regardless of the underlying block membership, and meaningful separation between the two blocks begins to emerge only when considering higher-order embeddings. Therefore, approaches restricted to first-order information, such as~\cite{gallagher2023spectral}, are inherently incapable of discriminating between communities whose structure is encoded in the higher-order moments of edge weights.

\subsection{Accuracy of moment recovery with varying number of nodes}

Accurate estimation of higher-order moments is known to be highly sensitive to the amount of data available. In particular, the variance of moment estimators tends to grow rapidly with the order of the moment, requiring markedly larger sample sizes to obtain stable estimates \cite{bourin1998efficiency}. This is because higher-order moments are dominated by extreme values in the data, making them especially prone to noise and outliers. Admittedly, this a challenge facing the proposed model.

To illustrate how this behavior affects the WRDPG model, we simulate a two-block weighted SBM with \( N = 2000 \) nodes, where 70\% of them are assigned to community 1. The interconnection probabilities are given by the matrix
\[
\bbB = \begin{pmatrix}
0.7 & 0.3\\
0.3 & 0.5
\end{pmatrix},
\]
and edge weights are sampled from a Gaussian distribution with mean $\mu=1$ and standard deviation $\sigma=0.5$.

Figure~\ref{fig:embeddings_sbm_N=2000} depicts the estimated latent position matrices \( \hbX[k] \) corresponding to moments \( k = 1 \) and \( k = 4 \). To assess the quality of the embeddings, we compare the entries of the empirical moment matrices \( \hbM[k] = \hbX[k] \hbX^\top[k] \) with the true moments. The results, shown in the first two columns of Figure \ref{fig:embeddings_sbm_N=2000}, indicate that the embeddings closely follow a mixture of multivariate Gaussian distributions and that, as expected, the accuracy of \( \hbM[k] \) degrades as the moment order increases.

\begin{figure}[t]
    \centering
        \includegraphics[width=0.48\linewidth]{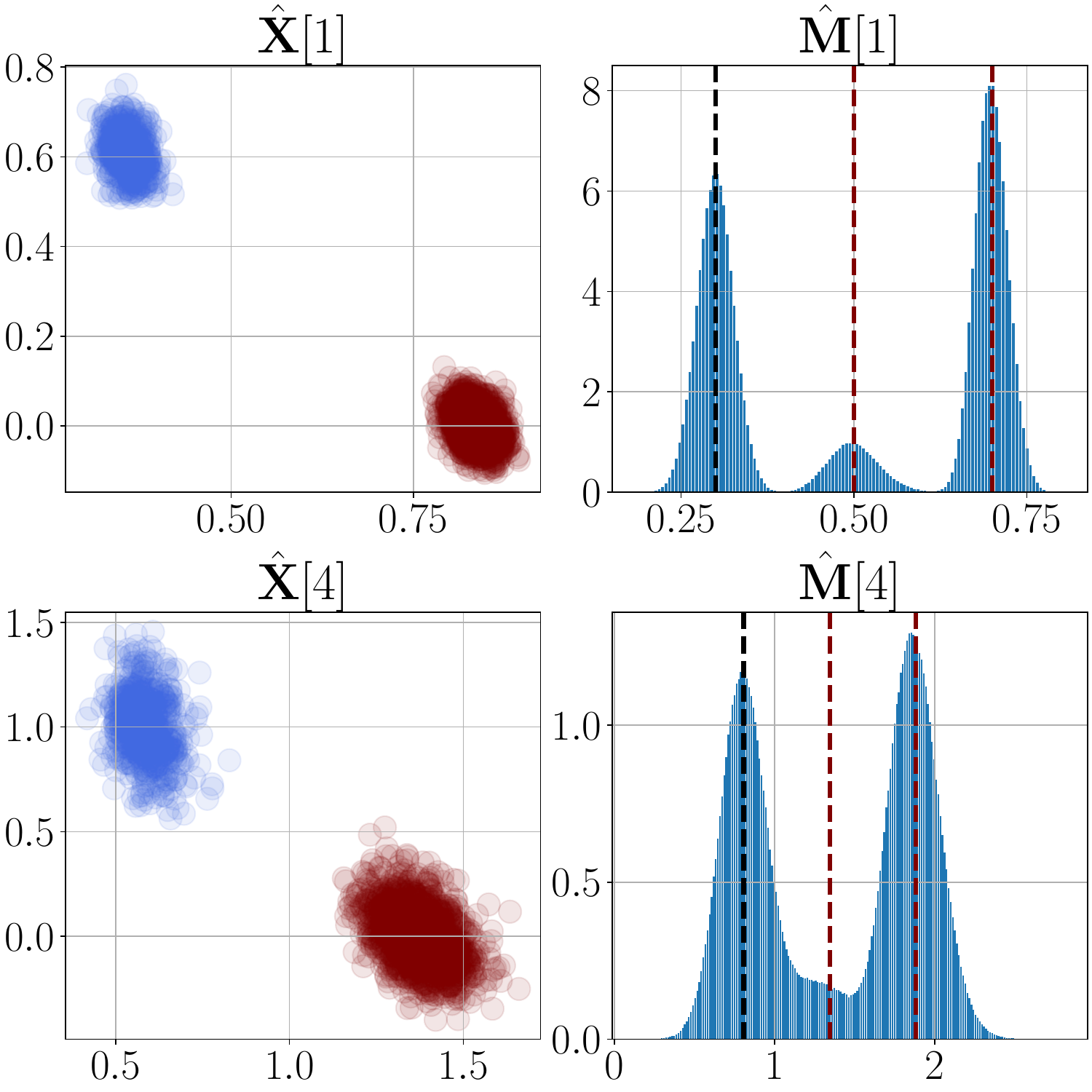}
        \includegraphics[width=0.48\linewidth]{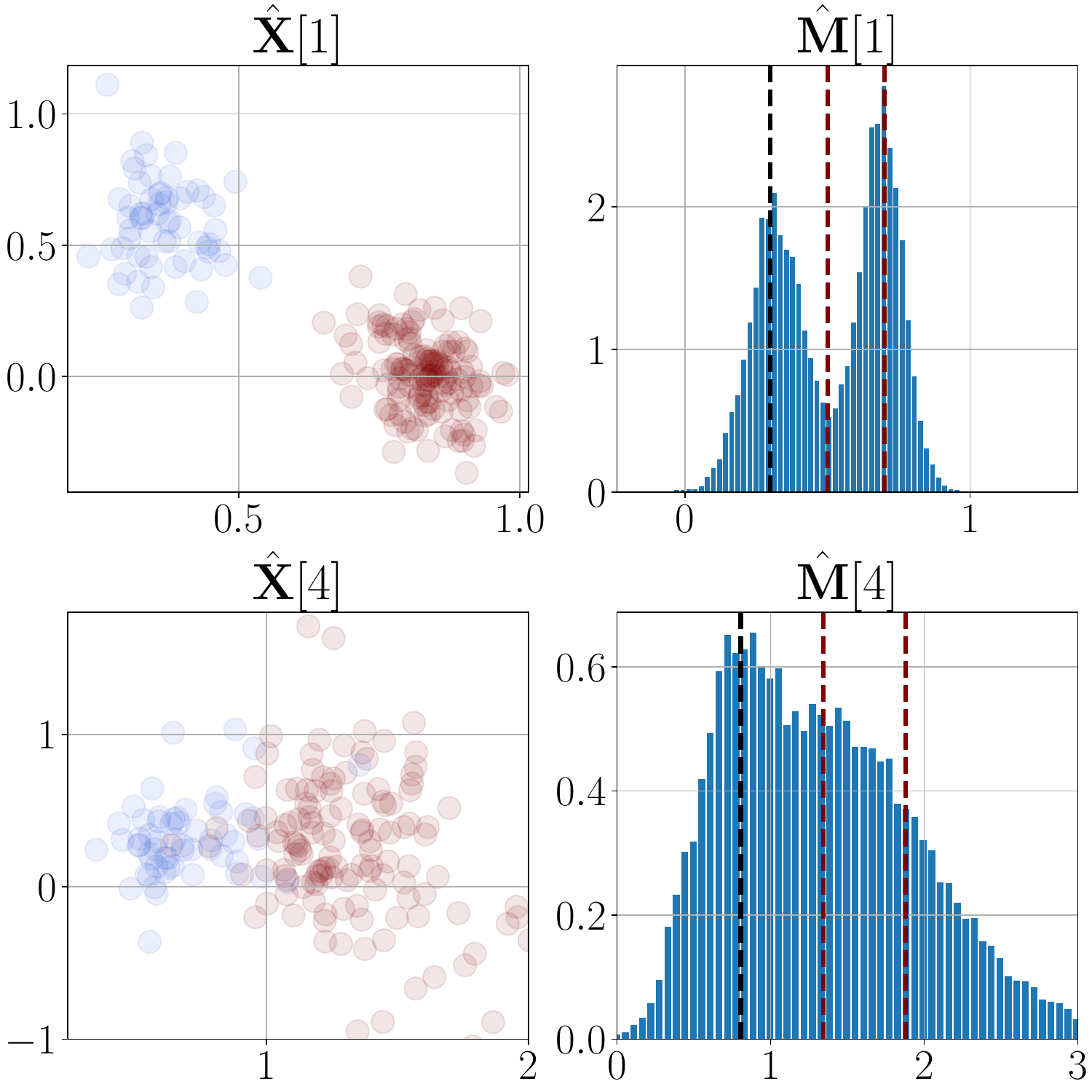}
    \caption{Inference results for a two-class SBM with Gaussian weights and $N=2000$ nodes (first and second columns) and $N=200$ nodes (third and fourth columns). The plots on the second and fourth columns show histograms of the estimated $\hbM[k]$ and the vertical lines indicate the true moments. For $N=2000$ embeddings and moments are accurately estimated up to order $k=4$, while accuracy degrades in the $N=200$ setting. Also, for fixed sample size performance degrades as the order increases from $k=1$ (top row) to $k=4$ (bottom row).}
    \label{fig:embeddings_sbm_N=2000}
\end{figure}

This observation is further reinforced in a second experiment presented in the third and fourth columns of Figure \ref{fig:embeddings_sbm_N=2000}, where the number of nodes (i.e., the sample size) is reduced to \( N = 200 \). Although the embeddings corresponding to \( k = 1 \) remain fairly accurate, the limited sample size noticeably impairs the estimation for \( k = 4 \), illustrating the practical limitations imposed by high-order moment estimation. These examples underscore the importance of accounting for finite-sample effects when working with moment-based estimators in network models such as WRDPG.
    
\section{Asymptotic results}
\label{sec:asymptotic_results}

Recall the ASE estimator introduced in Section \ref{ssec:estimation}.  Here we establish the asymptotic results that, for a fixed index $k$, characterize the behavior of the estimated latent positions $\hbX[k]$ when the number of nodes $N$ goes to infinity. Given the inherent rotational ambiguity in the WRDPG model, the latent position sequence is estimable up to an unknown orthogonal transformation. Thus, our main results (consistency and asymptotic normality) are stated in terms of a sequence of orthogonal matrices $\bbQ_k \in O(d)$.

For these results to hold, we make the following two assumptions. Here, $F$ is an inner-product distribution as in Definition \ref{def:inner_product_dist}, $\bbX_k := \{\bbX[k]\}_{k\geq 0}$ is a sequence of latent positions matrices, and $\bbW$ is the adjacency matrix of a WRDPG graph, i.e., $(\bbW,\bbX_k) \sim \mathrm{WRDPG}(F)$ as in Definition \ref{def:wrdpg_model}.

\begin{assumption}\label{assumption:full_Rank}
Let $\{\bbx[k]\}_{k\geq 0} \sim F$. Then for each $k$, the second moment matrix $\bbDelta_k = \E{\bbx[k]\bbx^\top[k]}$ has full rank $d$.
\end{assumption}

\begin{assumption}[Sub-Weibull weights]\label{assumption:bounded_rv}
There exists a constant $\theta > 0$ such that for each pair $1 \leq i<j \leq N$, there is a constant $C_{ij}>0$ satisfying:
\begin{displaymath}
\P{|W_{ij}| \geq t | \bbX_k} \leq 2 \exp\left(-\left(\frac{t}{C_{ij}}\right)^{1/\theta}\right) \text{ for all } t \geq 0.
\end{displaymath}
\end{assumption}

\begin{remark}\label{remark:weibull}
Since the class of sub-Weibull random variables (rvs) is closed under multiplication \cite[Proposition 2.3]{vladimirova2020sub}, it follows that $W_{ij}^k$ is also sub-Weibull, with parameter $k\theta$. This will allow us to establish our asymptotic results by invoking a concentration bound for sums of sub-Weibull variables, which is proven in Proposition~\ref{prop:concentration_sub_weibull}.

\end{remark}

\subsection{Consistency}
\label{sec:consistency}

We first establish the consistency of the estimated latent positions (up to an unknown orthogonal transformation). 

\begin{theorem}\label{theorem:consistency}
Let $F$ be an inner product distribution satisfying Assumption \ref{assumption:full_Rank} and consider $(\bbW,\bbX_k) \sim \mathrm{WRDPG}(F)$ satisfying Assumption \ref{assumption:bounded_rv}. Then, for each index $k$ there exists an orthogonal matrix $\bbQ_k \in O(d)$ such that

\begin{displaymath} 
    \twoToInf{\hbX[k] \bbQ_k - \bbX[k]} = \OP{N^{-1/2}\log^{k\theta}{N}} .
\end{displaymath}
\end{theorem}

Theorem~\ref{theorem:consistency} generalizes a previous consistency result for the WRDPG model~\cite[Theorem 1]{marenco2021tsipn} in two key aspects. First, it extends the result to accommodate unbounded edge weights. Second, it provides a bound on the difference between the estimated and actual latent positions (up to an unknown orthogonal transformation) in terms of the $2\to \infty$ norm rather than the Frobenius norm. As noted in \cite{cape2019twoToInfinity}, the $2\to \infty$ norm offers  tighter control over the estimation error. This can be seen from the identity
\begin{displaymath}
    \twoToInf{\bbA} = \max_{1\leq i \leq n} \twonorm{\bba_{i}^\top},
\end{displaymath}
where $\bba_{i}^\top$ denotes the $i$-th row of $\bbA\in \reals^{n\times m}$. This means that the $2\to \infty$ norm corresponds to the maximum Euclidean norm of the rows of $\bbA$, so Theorem \ref{theorem:consistency} bounds the maximum error between the estimated and true latent vectors for \emph{any} node in a WRDPG graph. In contrast, controlling the Frobenius norm does not guarantee such a uniform bound.

Our consistency result aligns with previous findings for the vanilla RDPG model \cite{priebe2018survey,lyzinski2017community}, in that we provide a bound for the $2\to \infty$ norm. It also generalizes the consistency result in \cite{gallagher2023spectral}, as their theorem applies only to $\hbX[1]$, meaning they establish consistency solely for the estimated latent positions associated with the mean. Nevertheless, our proof follows their methodology, adapting the necessary arguments to accommodate our more general setup. We note that their proof scheme closely follows that of the RDPG model, as first established in \cite[Theorem 15]{lyzinski2017community}. Additionally, it is worth mentioning that a more general result on matrix perturbations can be obtained using similar techniques; see \cite{cape2019twoToInfinity}. 

As in the unweighted case, the proof of Theorem \ref{theorem:consistency} relies on expressing the difference between $\hbX[k] \bbQ_k$ and $\bbX[k]$ as the sum of a dominant term and a series of remainder terms. The following proposition shows that these remainders are of a lower order than the dominant term. Its proof is deferred to Appendix \ref{app:technical_lemmas}.

\begin{proposition}\label{prop:remainders}
    Let $(\bbW,\bbX_k) \sim \mathrm{WRDPG}(F)$ as in Theorem \ref{theorem:consistency}. For each fixed integer $k\geq 0$, let $\bbW_k := \bbW^{(k)}$. Also let $\bbM_k := \bbX[k]\bbX^\top[k]$ and denote its spectral decomposition as $\bbM_k = \bbU_k\bbD_k\bbU_k^\top$, with $\bbD_k \in \reals^{d\times d}$ and $\bbU_k \in \reals^{N\times d}$. For $\bbS_k \in O(d)$,  define $\bbR_{k_1}$, $\bbR_{k_2}$, $\bbR_{k_3}$ and $\bbR_{k_4}$ as
    \begin{align*}
        \bbR_{k_1} &= \bbU_k\left(\bbU_k^\top \hbU_k\hbD_k^{1/2} -\bbD_k^{1/2} \bbS_k  \right),\\
        \bbR_{k_2} &=\left(\bbI- \bbU_k\bbU_k^\top \right)\left(\bbW_k- \bbM_k \right)\left(\hbU_k -\bbU_k\bbS_k  \right)\hbD_k^{-1/2},\\
        \bbR_{k_3} &=- \bbU_k\bbU_k^\top\left(\bbW_k- \bbM_k \right)\bbU_k\bbS_k  \hbD_k^{-1/2},\\
        \bbR_{k_4} &=\left(\bbW_k- \bbM_k \right)\bbU_k\left(\bbS_k\hbD_k^{-1/2} -\bbD_k^{-1/2}\bbS_k  \right).
    \end{align*}
    Then there exists a choice of $\bbS_k \in O(d)$ such that the following holds:
    \begin{align*}
        \twoToInf{\bbR_{k_1}} &= \OP{N^{-1} \log^{k\theta} N}, \\
        \twoToInf{\bbR_{k_2}} &= \OP{N^{-3/2} \log^{2k\theta}  N},\\
        \twoToInf{\bbR_{k_3}} &= \OP{N^{-1} \log^{k\theta} N},\\
        \twoToInf{\bbR_{k_4}} &= \OP{N^{-3/2} \log^{2k\theta} N}.
    \end{align*}
\end{proposition}

Using this proposition, we can prove Theorem~\ref{theorem:consistency}.
\begin{proof}[\textbf{Proof of Theorem~\ref{theorem:consistency}}]
     First, assume $\bbX[k] = \bbU_k\bbD_k^{1/2}$. Fix $\bbS_k \in O(d)$ such that Proposition \ref{prop:remainders} holds. Since $\hbX[k] = \hbU_k\hbD_k^{1/2}$, we have that:
    \begin{align*}
        \hbX[k] - \bbX[k]\bbS_k =&\, \hbU_k\hbD_k^{1/2} - \bbU_k\bbD_k^{1/2}\bbS_k \\
            =& \hbU_k\hbD_k^{1/2} - \bbU_k\bbU_k^\top\hbU_k\hbD_k^{1/2} +\bbU_k\left(\bbU_k^\top\hbU_k\hbD_k^{1/2} - \bbD_k^{1/2}\bbS_k \right)\\
        =&\, \hbU_k\hbD_k^{1/2} - \bbU_k\bbU_k^\top\hbU_k\hbD_k^{1/2} + \bbR_{k_1}.
    \end{align*}
    Note that $\hbU_k\hbD_k=\bbW_k\hbU_k$, which implies that $\hbU_k\hbD_k^{1/2}=\bbW_k\hbU_k\hbD_k^{-1/2}$. So:
    \begin{align}
        \hbX[k] - \bbX[k]\bbS_k =&\, \bbW_k\hbU_k\hbD_k^{-1/2} - \bbU_k\bbU_k^\top\bbW_k\hbU_k\hbD_k^{-1/2} + \bbR_{k_1} \nonumber\\
            =&\,   \bbW_k\hbU_k\hbD_k^{-1/2} - \bbM_k\hbU_k\hbD_k^{-1/2}  +\bbM_k\hbU_k\hbD_k^{-1/2} - \bbU_k\bbU_k^\top\bbW_k\hbU_k\hbD_k^{-1/2} + \bbR_{k_1} \nonumber\\
            =&\,  \left( \bbW_k - \bbM_k\right) \hbU_k\hbD_k^{-1/2}  + \bbU_k\bbU_k^\top\bbM_k \hbU_k\hbD_k^{-1/2} - \bbU_k\bbU_k^\top\bbW_k\hbU_k\hbD_k^{-1/2} + \bbR_{k_1} \label{eq:masajeo1}\\
        =& \left(\bbI - \bbU_k\bbU_k^\top \right) \left(  \bbW_k - \bbM_k \right)\hbU_k\hbD_k^{-1/2} + \bbR_{k_1}, \label{eq:masajeo2}
    \end{align}
    where in~\eqref{eq:masajeo1} we have used the fact that $\bbM_k=\bbU_k\bbU_k^\top\bbM_k $. From \eqref{eq:masajeo2} we find that:
    \begin{align}
        \hbX[k] - \bbX[k]\bbS_k =&\, \left(\bbI - \bbU_k\bbU_k^\top \right) \left(  \bbW_k - \bbM_k \right)\left(\bbU_k\bbS_k - \bbU_k\bbS_k + \hbU_k \right)\hbD_k^{-1/2} + \bbR_{k_1} \nonumber\\
        =&\, \left(  \bbW_k - \bbM_k \right)\bbU_k\bbS_k\hbD_k^{-1/2} +\bbR_{k_1}+\bbR_{k_2}+\bbR_{k_3} \nonumber\\
        \begin{split}
        =&\, \left(  \bbW_k - \bbM_k \right)\bbU_k\left(\bbD_k^{-1/2}\bbS_k + \bbS_k\hbD_k^{-1/2} -\bbD_k^{-1/2}\bbS_k\right) +\\
        &\, \quad +\bbR_{k_1}+\bbR_{k_2}+\bbR_{k_3}
        \end{split} \nonumber\\
        =&\, \left(  \bbW_k - \bbM_k \right)\bbU_k\bbD_k^{-1/2}\bbS_k + \bbR_{k_1}+\bbR_{k_2}+\bbR_{k_3} + \bbR_{k_4}. \label{eq:masajeo_restos}
    \end{align}

    Using Proposition \ref{prop:remainders}, we conclude that
    \begin{align*}
        \twoToInf{\hbX[k] - \bbX[k]\bbS_k} \leq &\, \twoToInf{ \left(  \bbW_k - \bbM_k \right)\bbU_k\bbD_k^{-1/2}\bbS_k} +  \OP{N^{-1}\log^{k\theta}(N)}.
    \end{align*}
    Since $\twoToInf{\bbA\bbB} \leq \twoToInf{\bbA} \spectral{\bbB}$, we have:
    \begin{align}
        \twoToInf{\hbX[k] - \bbX[k]\bbS_k} \leq  &\,  \twoToInf{ \left(  \bbW_k - \bbM_k \right)\bbU_k} \spectral{\bbD_k^{-1/2}} \, + \OP{N^{-1}\log^{k\theta}(N)} \nonumber\\
        \leq &\, \twoToInf{ \left(  \bbW_k - \bbM_k \right)\bbU_k} \OP{N^{-1/2}} \, + \OP{N^{-1}\log^{k\theta}(N)}, \label{eq:masajeo3}
    \end{align}
    where for the second inequality we have used that $\twonorm{\bbD_k^{-1/2}}=\OP{N^{-1/2}}$ due to the fact that the top $d$ eigenvalues of $\bbM_k$ are $\EquivP{N}$; see Appendix \ref{app:eigenvalues}.
    
    To bound the term  $\twoToInf{ \left(  \bbW_k - \bbM_k \right)\bbU_k}$, we first note that each entry of that matrix is a linear combination of centered, sub-Weibull rvs. Indeed,
    \begin{align*}
        \left[\left(  \bbW_k - \bbM_k \right)\bbU_k\right]_{ij} = \sum_{l=1}^N \left(  W_{il} - M_{il} \right)U_{lj} = \sum_{l\neq i} \left(  W_{il} - M_{il} \right)U_{lj} - M_{ii}U_{ii},
    \end{align*}
    where we denote the $(i,j)$ element of $\bbW_k$, $\bbM_k$, and $\bbU_k$ by $W_{ij}$, $M_{ij}$, and $U_{ij}$, respectively. Since the sum in the rightmost term consists of centered, independent sub-Weibull rvs with parameter $k\theta$, we can apply a concentration result for this setting that we prove in Proposition~\ref{prop:concentration_sub_weibull} of Appendix~\ref{app:technical_lemmas}. Using this result, we find that for all sufficiently large $t$,
    \begin{align*}
        \P{\Big|\left[\left(  \bbW_k - \bbM_k \right)\bbU_k\right]_{ij} \Big| \geq t} \leq 2\exp\left(- \frac{ct^2}{\displaystyle L_k^2 \sum_{l\neq i} |U_{il}|^2} \right),
    \end{align*}
    for some constant $K_2>0$ depending on $k\theta$. By choosing $t=K_2\log^{k\theta}N$ we find that $\left[\left(  \bbW_k - \bbM_k \right)\bbU_k\right]_{ij} $ is $\OP{\log^{k\theta} N}$, and by summing over $j=1,\dots,d$ we have that each row of $\left(  \bbW_k - \bbM_k \right)\bbU_k$ is also $\OP{\log^{k\theta} N}$, so its $2\to\infty$ norm is of the same order.
    
    In Appendix \ref{app:eigenvalues} we show that Assumption \ref{assumption:full_Rank} implies that $\lambda_i \left( \bbM_k \right) = \EquivP{N}$ for all $i=1,\dots, d$. Equation \eqref{eq:masajeo3} then implies that
    \begin{align*}
        \twoToInf{\hbX[k] - \bbX[k]\bbS_k} &= \OP{N^{-1/2}\log^{k\theta} N} \, + \OP{N^{-1}\log^{k\theta}(N)}\\
        &= \OP{N^{-1/2}\log^{k\theta}(N)}.
    \end{align*}
    Choosing $\bbQ_k = \bbS_k^{\top}$ shows that $\twoToInf{\hbX[k] \bbQ_k- \bbX[k]} =\OP{N^{-1/2}\log^{k\theta}(N)}$, as desired.
    
    If $\bbX[k] \neq \bbU_k\bbD_k^{1/2}$, we have that $\bbX[k]\bbT_k = \bbU_k\bbD_k^{1/2}$ for some $\bbT_k \in O(d)$. Note that the above calculations imply that
    \begin{align*}
         \twoToInf{\hbX[k] -  \bbU_k\bbD_k^{1/2}\bbS_k} &= \OP{N^{-1/2}\log^{k\theta} N}.
    \end{align*}
   Therefore, in this case it is enough to choose $\bbQ_k = \bbS_k^\top\bbT_k^\top \in O(d)$, since
   \begin{align*}
       \twoToInf{\hbX[k]\bbQ_k -  \bbX[k]} &= \twoToInf{( \hbX[k] -  \bbX[k]\bbT_k\bbS_k) \bbQ_k} \leq \twoToInf{\hbX[k] -  \bbU_k\bbD_k^{1/2}\bbS_k} .
   \end{align*}
\end{proof}

\subsection{Asymptotic Normality}

Next, we show that latent positions behave, asymptotically as $N\to\infty$, as multivariate normal random variables. We also explicitly calculate the covariance matrix of such a normal in terms of the second-moment matrix $\bbDelta_k$ of the latent positions and the latent positions themselves.

\begin{theorem}\label{theorem:tcl}
Let  $(\bbW,\bbX_k) \sim \mathrm{WRDPG}(F)$ be as in Theorem \ref{theorem:consistency}. For each index $k$, define the variance function $v_k:\reals^d \times \reals^d \mapsto \reals$ as
\begin{align*}
    v_k(\bbx,\bby) := \var{W_{ij}^k \big| \bbx_i[k]=\bbx, \bbx_j[k] = \bby},
\end{align*}
where $W_{ij}$ is the $(i,j)$ entry of $\bbW$. Let $\bbSigma_k:\reals^d \mapsto \reals^{d\times d}$ be the covariance function
\begin{align*}
    \bbSigma_k(\bbx) = \bbDelta_k^{-1} \E{v_k(\bbx,\bby_k)\bby_k\bby_k^\top} \bbDelta_k^{-1},
\end{align*}
where $\{\bby_k\}_{k\geq 0} \sim F$ and $\bbDelta_k $ is the second-moment matrix $\bbDelta_k = \E{\bby_k\bby_k^\top}$.

Then for each $k$ there exists a sequence of orthogonal matrices $ \{\bbQ_{k_N}\}_{N \geq 0}$ such that for all $\bbz \in \reals^d$ and for any fixed row index $i$,
\begin{align*}
    \lim_{N\to \infty} \P{N^{1/2}\left( \hbX[k]\bbQ_{k_N} -  \bbX[k] \right)_i^\top \leq \bbz\, \Big| \, \bbx_i[k] = \bbx} = \bbPhi(\bbz;\bbSigma_k(\bbx)),
\end{align*}
where $\bbPhi(\cdot;\bbSigma)$ stands for the cumulative distribution function of a $\ccalN(\mathbf{0},\bbSigma)$ random vector.
\end{theorem}

\begin{proof}
    From the proof of Theorem \ref{theorem:consistency}, we have that for each fixed $k$ there exists $\bbS_k, \bbT_k \in O(d)$ such that for $\bbQ_k = \bbS_k^\top\bbT_k^\top $ it holds
    \begin{align*}
        \hbX[k]\bbQ_k - \bbX[k] =  \left(\hbX[k] - \bbU_k \bbD_k^{1/2} \bbS_k\right)\bbQ_k .
    \end{align*}
    Also from that proof [see \eqref{eq:masajeo_restos}], we have that
    \begin{align*}
        \hbX[k] - \bbU_k \bbD_k^{1/2} \bbS_k = \left(  \bbW_k - \bbM_k \right)\bbU_k\bbD_k^{-1/2}\bbS_k + \bbR_k,
    \end{align*}
    where we have defined $\bbR_k = \bbR_{k_1} + \bbR_{k_2} + \bbR_{k_3} + \bbR_{k_4}$. Therefore:
    \begin{align}
        N^{1/2} \left(\hbX[k]\bbQ_k - \bbX[k] \right)= N^{1/2}\left(  \bbW_k - \bbM_k \right)\bbU_k\bbD_k^{-1/2}\bbT_k^\top +N^{1/2}\bbR_k\bbQ_k. \label{eq:ase_masajeo_1}
    \end{align}
    Using Proposition \ref{prop:remainders}, we have that $\twoToInf{N^{1/2}\bbR_k\bbQ_k} \to 0$ as $N\to\infty$. Therefore, we can focus on the first summand in \eqref{eq:ase_masajeo_1}.
    
    Recall from the proof of Theorem \ref{theorem:consistency} that $\bbT_k$ is such that $\bbX[k]\bbT_k = \bbU_k\bbD_k^{1/2}$, which implies that $\bbU_k\bbD_k^{-1/2} = \bbX[k]\bbT_k\bbD_k^{-1}$. Thus:
    \begin{align*}
        N^{1/2}\left(  \bbW_k - \bbM_k \right)\bbU_k\bbD_k^{-1/2}\bbT_k^\top = N^{1/2 }\left(  \bbW_k - \bbM_k \right)\bbX[k]\bbT_k\bbD_k^{-1}\bbT_k^\top .
    \end{align*}
    Therefore, the $i$-th row of the first summand in \eqref{eq:ase_masajeo_1} equals:
    \begin{align}
        N^{1/2}\left[\left(  \bbW_k - \bbM_k \right)\bbU_k\bbD_k^{-1/2}\bbT_k^\top\right]_i^\top &= N^{1/2}\bbT_k\bbD_k^{-1}\bbT_k^\top  \left\lbrace\left(  \bbW_k - \bbM_k \right)\bbX[k]\right\rbrace_i^\top \nonumber\\
        &= N \bbT_k\bbD_k^{-1}\bbT_k^\top\left[N^{-1/2}\sum_{j=1}^N \left(  W_{ij}^k - M_{ij} \right)\bbx_j[k]\right] \nonumber\\
        \begin{split}
        &= N \bbT_k\bbD_k^{-1}\bbT_k^\top\left[N^{-1/2}\sum_{j\neq i} \left(  W_{ij}^k - M_{ij} \right)\bbx_j[k]\right] \\
         &\,\quad - N \bbT_k\bbD_k^{-1}\bbT_k \left(N^{-1/2} M_{ii} \bbx_i[k]\right),
        \end{split} \label{eq:ase_masajeo_2}
    \end{align}
    where as before we denote the $(i,j)$ element of $\bbW$ and $\bbM_k$ by $W_{ij}$ and $M_{ij}$, respectively. Conditioning on $\bbx_i[k] = \bbx$, we have that for the second term in \eqref{eq:ase_masajeo_2}
    \begin{align*}
        \twonorm{N \bbT_k\bbD_k^{-1}\bbT_k \left(N^{-1/2} M_{ii} \bbx_i[k]\right)} &\leq N^{1/2}|M_{ii}|\,\twoToInf{\bbT_k\bbD_k^{-1}\bbT_k} \spectral{\bbx}\\
        &\leq N^{1/2}L_k \spectral{\bbD_k^{-1}} \,\spectral{\bbx},
    \end{align*}
    where in the second inequality we have used that Assumption \ref{assumption:bounded_rv} implies that $|M_{ij}|<L_k$ for all $(i,j)$, and that $\twoToInf{\bbA}\leq \spectral{\bbA}$. Because $\spectral{\bbD_k^{-1}}  = \OP{N^{-1}}$, we conclude that the second term in \eqref{eq:ase_masajeo_2} is $\OP{N^{-1/2}}$.

    As for the first term in \eqref{eq:ase_masajeo_2}, conditional on $\bbx_i[k] = \bbx$ we have that $M_{ij} = \bbx^\top\bbx_j[k]$, so the terms in the sum
    \begin{align*}
        N^{-1/2}\sum_{j\neq i} \left(  W_{ij}^k - M_{ij} \right)\bbx_j[k]
    \end{align*}
    are centered, independent random variables, whose covariance matrix is
    \begin{align*}
        \tbSigma_k(\bbx) = \E{v_k(\bbx,\bby_k)\bby_k\bby_k^\top},
    \end{align*}
    where $\{\bby_k\}_{k\geq 0} \sim F$. Therefore, the multivariate central limit theorem implies that
    \begin{align*}
          N^{-1/2}\sum_{j\neq i} \left(  W_{ij} - M_{ij} \right)\bbx_j[k] \xrightarrow{\mathcal{L}} \ccalN (\bb0,\tbSigma_k(\bbx)).
    \end{align*}
    To conclude the proof, recall that $\bbT_k\in O(d)$ is such that $\bbX[k]\bbT_k = \bbU_k\bbD_k^{1/2}$, so $\bbX[k] = \bbU_k\bbD_k^{1/2}\bbT_k^\top$ and therefore $\bbX^\top[k]\bbX[k] = \bbT_k \bbD_k \bbT_k^\top$. This implies that $$\left(\bbX^\top[k]\bbX[k]\right)^{-1} = \bbT_k \bbD_k^{-1} \bbT_k^\top,$$ and therefore by the law of large numbers we have that almost surely it holds
    \begin{align*}
        N \bbT_k\bbD_k^{-1}\bbT_k^\top \to \bbDelta_k^{-1}.
    \end{align*}
    This implies that the first term in \eqref{eq:ase_masajeo_2} converges in distribution to a multivariate normal $\ccalN(\bb0,\bbSigma(\bbx))$, which completes the proof. 
    
\end{proof}

\begin{remark}
    The covariance function in Theorem \ref{theorem:tcl} depends on $v_k(\bbx,\bby)$, which is the conditional variance of $W_{ij}^k$ given $\bbx_i[k]=\bbx$ and $\bbx_j[k] = \bby$. Therefore, we can rewrite $v_k$ as
    \begin{align*}
        v_k(\bbx,\bby) &= \E{(W_{ij}^k)^2\big| \bbx_i[k]=\bbx, \bbx_j[k] = \bby} - \E{W_{ij}^k\big| \bbx_i[k]=\bbx, \bbx_j[k] = \bby}^2\\
        &=\E{W_{ij}^{2k}\big| \bbx_i[k]=\bbx, \bbx_j[k] = \bby} - (\bbx^\top\bby)^2,
    \end{align*}
    since the WRDPG model by definition imposes $\E{W_{ij}^k\big| \bbx_i[k]=\bbx, \bbx_j[k] = \bby} = \bbx^\top\bby$. The expectation of $W_{ij}^{2k}$ need not have a closed-form expression given $\bbx_i[k]=\bbx$ and $ \bbx_j[k] = \bby$. However, if we further condition on the events $\bbx_i[2k] = \bbx_2$ and  $\bby_i[2k] = \bby_2$, we have that $\E{W_{ij}^{2k}} = \bbx_2^\top \bby_2$. Therefore, by conditioning on $\bbx_i[k] = \bbx_1$ and $\bbx_i[2k] = \bbx_2$ we have the following Corollary of Theorem \ref{theorem:tcl}.
\end{remark}
\begin{corollary}
Let  $(\bbW,\bbX_k) \sim \mathrm{WRDPG}(F)$ be as in Theorem \ref{theorem:consistency}. Then for each $k$ there exists a sequence of orthogonal matrices $ \{\bbQ_{k_N}\}_{N \geq 0}$ such that for all $\bbz \in \reals^d$ and for any fixed row index $i$,
\begin{align*}
    \lim_{N\to \infty} \P{N^{1/2}\left( \hbX[k]\bbQ_{k_N} -  \bbX[k] \right)_i^\top \leq \bbz\, \Big| \, \bbx_i[k] = \bbx_1, \, \bbx_i[2k] = \bbx_2} = \bbPhi(\bbz,\bbSigma_k(\bbx_1,\bbx_2)),
\end{align*}
where $\bbSigma_k:\reals^d \times \reals^d \mapsto \reals^{d\times d}$ is the covariance function
\begin{align*}
    \bbSigma_k(\bbx_1,\bbx_2) = \bbDelta_k^{-1} \E{\left(\bbx_2^\top\bby_{2k} - (\bbx_1^\top\bby_k)^2\right)\bby_k\bby_k^\top} \bbDelta_k^{-1},
\end{align*}
and $\{\bby_k\}_{k\geq 0} \sim F$.

\end{corollary}

As discussed in Section~\ref{ex:wsbm}, when the random graph follows a weighted SBM  with $C$ communities, for each moment index $k$ there are at most $C$ distinct latent positions that a node can assume. We denote these positions by $\bby_{m}[k]$, for $m = 1, \dots, C$. Assume that each node belongs to community $m$ with probability $\pi_{m}$, and define the vector $\bbpi = [\pi_1, \pi_2, \dots, \pi_C]^\top$. Let $\bbB \in \reals^{C \times C}$ be the matrix of edge-formation probabilities between communities -- i.e.,  entry $b_{lm}$ specifies the probability that an edge exists between a node in community $l$ and a node in community $m$. Let $d_{lm}$ denote the probability distribution governing the edge weights between nodes in communities $l$ and $m$, with $m_{lm}[k]$ denoting its $k$-th moment.

Under these conventions, each entry of the weighted adjacency matrix $\bbW$ is either zero with probability $1 - \bbpi^\top \bbB \bbpi$, or drawn from $d_{lm}$ with probability $\pi_l \pi_m b_{lm}$. In this setup, the following Corollary of Theorem \ref{theorem:tcl} holds.

\begin{corollary}\label{corollary:tcl_sbm}
    With the above notation, for each moment index $k$ there exists a sequence of orthogonal matrices $ \{\bbQ_{k_N}\}_{N \geq 0}$ such that for all $\bbz \in \reals^d$ and for any fixed row index $i$,
\begin{align*}
    \lim_{N\to \infty} \P{N^{1/2}\left( \hbX[k]\bbQ_{k_N} -  \bbX[k] \right)_i^\top \leq \bbz\, \Big| \, \text{node } i \text{ belongs to community } l } = \bbPhi(\bbz,\bbSigma_{kl}),
\end{align*}
where the covariance matrix $\bbSigma_{kl}$ is given by $\bbSigma_{kl} = \bbDelta_k^{-1}\tbSigma_{kl} \bbDelta_k^{-1}$, with
\begin{align*}
    \tbSigma_{kl} = \sum_{m=1}^C \pi_m  \left(b_{lm}m_{lm}[2k] - b_{lm}^2m_{lm}^2[k] \right) \bby_m[k]\bby_m^\top[k]
\end{align*}
and
\begin{align*}
    \bbDelta_{k} = \sum_{m=1}^C \pi_m \bby_m[k]\bby_m^\top[k].
\end{align*}
\end{corollary}

\begin{remark}(Sparse WRDPG)
    The WRDPG model may be extended to sparse weighted networks by introducing a sparsity factor $\rho_N^{(k)} \to 0$ for each moment index $k$, so that
    \begin{equation*}
         \E{W_{ij}^k} = \rho_N^{(k)}\bbx_i^\top[k] \bbx_j[k].
    \end{equation*}
    In this regime, the operator norm of the population moment matrix scales as $\spectral{\bbM_k} = \EquivP{N\rho_N^{(k)}}$.
    
    Under the sub-Weibull tail Assumption~\ref{assumption:bounded_rv}, the empirical $k$-th moment matrix satisfies
    \begin{equation*}
        \spectral{\bbW^{(k)} - \bbM_k} = \OP{\rho_N^{(k)} \log^{k\theta} N}.
    \end{equation*}
    Consequently, spectral consistency of the ASE for the $k$-th moment holds provided that the signal dominates the stochastic fluctuation, namely,
    \begin{equation*}
        N\rho_N^{(k)} \gg \log^{k\theta} N.
    \end{equation*}
    For asymptotic normality of the estimated latent positions, a stronger sparsity condition is required. In particular, normal fluctuations are expected to hold whenever
    \begin{equation*}
        N\rho_N^{(k)} \gg \log^{2k\theta} N,
    \end{equation*}
    ensuring that the central limit scaling dominates the residual spectral error. All in all, these sparsity requirements are analogous to those appearing in the sparse unweighted RDPG literature~\cite{tang2018connectome} and indicate that the WRDPG framework also admits sparse regimes with expected degrees growing at polylogarithmic rates.
\end{remark}

\section{Graph generation}
\label{sec:generative}

So far, we have formally defined the WRDPG model and demonstrated its versatility and discriminative power through examples. In addition, we have provided statistical guarantees for the proposed ASE-based estimation method. In this section, we investigate how to generate graphs from the WRDPG model with latent positions
\begin{displaymath}
\bbX[0], \bbX[1], \dots, \bbX[K]\in \reals^{N \times d},    
\end{displaymath}
where $\bbX[k] = [\bbx_1[k], \bbx_2[k],\dots, \bbx_N[k]]^\top$.
In other words, our aim is to sample the adjacency matrix $\bbW \in \reals^{N\times N}$, such that, for each pair $1\leq i < j \leq N$, $W_{ij}$ follows a distribution whose first $K+1$ moments are given by $\bbx_i^\top[k]\bbx_j[k]$, for all $k=0,1,\dots, K$. The latent positions could be defined so as to match a prescribed weight distribution, or, be estimated from real networks via the ASE. 

In the sequel, we first consider a discrete weight distribution with finite support. We show that the weights' probability mass function (pmf) can be obtained in closed form, by solving a linear system of equations with Vandermonde structure. As we will see, our method is capable of reproducing the original characteristics of the network. 
Next, we address the problem of generating samples from WRDPG graphs with continuous weight distributions. To this end, we rely on the maximum entropy principle to recover the probability density function (pdf) from its moments. We develop a new primal-dual method, providing better and more robust solutions than prior art~\cite{saad2019pymaxent}; see also the results in Appendix \ref{app:maximum_entropy}. Finally, we consider the case of a mixture distribution, simultaneously reproducing the connectivity structure of a real network and its weight distribution. Throughout, we provide supporting examples using synthetic and real data.

\subsection{Discrete weights distribution}
\label{sec:discrete_generation}

We begin by deriving a solution for the case where $W_{ij}$ has a discrete distribution. Suppose that the latent positions for each node are given, and $W_{ij}$ takes on $R+1$ (known) distinct values $v_0, v_1,\dots,v_R$ with (unknown) probabilities $p_0,p_1,\dots,p_R$.\footnote{Note that both $v_r$'s and $p_r$'s are dependent on $i,j$, but that dependence has been omitted to improve readability.} To generate a WRDPG-adhering graph, we need to estimate these probabilities from the latent position sequence. In this case, the moments for such a distribution can be computed as
\begin{align*}
    \E{W_{ij}^k} = \sum_{r=0}^R v_r^k p_r = v_0^k p_0+v_1^k p_1+\dots+v_R^k p_R = \bbx_i^\top[k] \bbx_j[k]:=m_{ij}[k],
\end{align*}
where the usual convention $0^0=1$ is used in case $v_r=0$, for some $r$.
We then obtain the following system of $K+1$ linear equations on the pmf $\bbp = [p_0,p_1,\dots,p_R]^\top \in [0,1]^{R+1}$:
\begin{align}
    \left\lbrace
    \begin{array}{c c c}
        \phantom{v_0}p_0 + \phantom{v_1} p_1+\dots+\phantom{v_R} p_R &=& m_{ij}[0]  \\
        v_0p_0 +v_1p_1+\dots+v_R p_R &=& m_{ij}[1]\\
        v_0^2p_0 + v_1^2p_1+\dots+v_R^2 p_R &=& m_{ij}[2]\\
        \vdots& &\vdots\\
        v_0^Kp_0 + v_1^Kp_1+\dots+v_R^K p_R &=& m_{ij}[K]\\
    \end{array}
    \right.
    \Leftrightarrow \bbV \bbp = \bbm,
\label{eq:vandermonde}
\end{align}
where $\bbm = [m_{ij}[0],m_{ij}[1],\dots,m_{ij}[R]]^\top \in \reals^{R+1}$ and $\bbV \in \reals^{(K+1)\times(R+1)}$ is the Vandermonde matrix
\begin{align*}
    \bbV = \left(
   \begin{array}{c c c c}
    1 & 1 & \ldots & 1\\
    v_0& v_1 & \ldots & v_R\\
    v_0^2& v_1^2 & \ldots & v_R^2\\
    \vdots & \vdots & \ddots & \vdots\\
    v_0^K& v_1^K & \ldots & v_R^K
   \end{array}
   \right) .
\end{align*}
Note that if $K=R$ then the system \eqref{eq:vandermonde} has a unique solution. So, in order to estimate the pmf associated with edge $(i,j)$ we need to prescribe as many moments as there are possible values for $W_{ij}$. Given the first $R$ moments of the distribution, we readily obtain the probabilities as ${\bbp} = \bbV^{-1}\bbm$.

While the linear system \eqref{eq:vandermonde} needs to be solved for every pair $(i,j)$, if the distributions associated with different edges share a common support $v_0, v_1, \dots, v_R$, then $\bbV$ (and thus $\bbV^{-1}$) are the same across those edges, making computation of $\bbp$ for the entire graph less demanding.

\begin{example}
To illustrate how the graph generative process works, we simulated a two-class weighted SBM network with \( N = 500 \) nodes (350 in community 1 and 150 in community 2), and block probability matrix  
\begin{align}
\bbB = \left(
\begin{array}{cc}
0.7 & 0.2 \\
0.2 & 0.5
\end{array}
\right),
\label{eq:sbm_block_probabilities_discrete_genetarion}
\end{align}  
where edge weights are sampled from a discrete distribution supported on the values \( 1, 2, \dots,\) \(10 \), with \( p_i = \frac{1}{18} \) for \( i \neq 5 \) and \( p_5 = \frac{1}{2} \). Using~\eqref{eq:sbm_normal_mean_latent_sequence}, we compute the analytical latent positions for each node, where \( m_d[k] \) denotes the \( k \)-th moment of the aforementioned discrete distribution.

For each edge, we compute the finite moment sequence \( m_{ij}[k] = \bbx_i^\top[k] \bbx_j[k] \), for \( k = 0, \dots, 10 \). In this setup, each edge weight can take on values in $ \{0, 1, 2, \dots,$ $ 10\} $, with probabilities \( p_0,\, (1-p_0)p_1,\, (1-p_0)p_2,\, \dots,\, (1-p_0)p_{10} \), where \( p_0 \) is the probability that nodes \( i \) and \( j \) are not connected. Specifically, \( p_0 = 1 - 0.7 = 0.3 \) if \( i \) and \( j \) both belong to community 1, \( p_0 = 1 - 0.5 = 0.5 \) if both are in community 2, and \( p_0 = 1 - 0.2 = 0.8 \) otherwise. Thus, edge weights are discrete random variables with $R+1=11$ possible values, which requires 11 moments to uniquely solve~\eqref{eq:vandermonde}.

Figure~\ref{fig:discrete_sbm} shows the distributions of various metrics (degree, betweenness centrality, geodesic distance) for 100 simulated networks generated using the above procedure, along with the same metrics computed on a single reference network drawn from the base model--namely, a two-block SBM with $\bbB$ as in~\eqref{eq:sbm_block_probabilities_discrete_genetarion}, and edge weights sampled from the above discrete distribution. Apparently, the metrics for the network generated from the base model align closely with the distribution of corresponding metrics computed from the sampled graphs.

\begin{figure}[t]
    \centering
    \includegraphics[width=\textwidth]{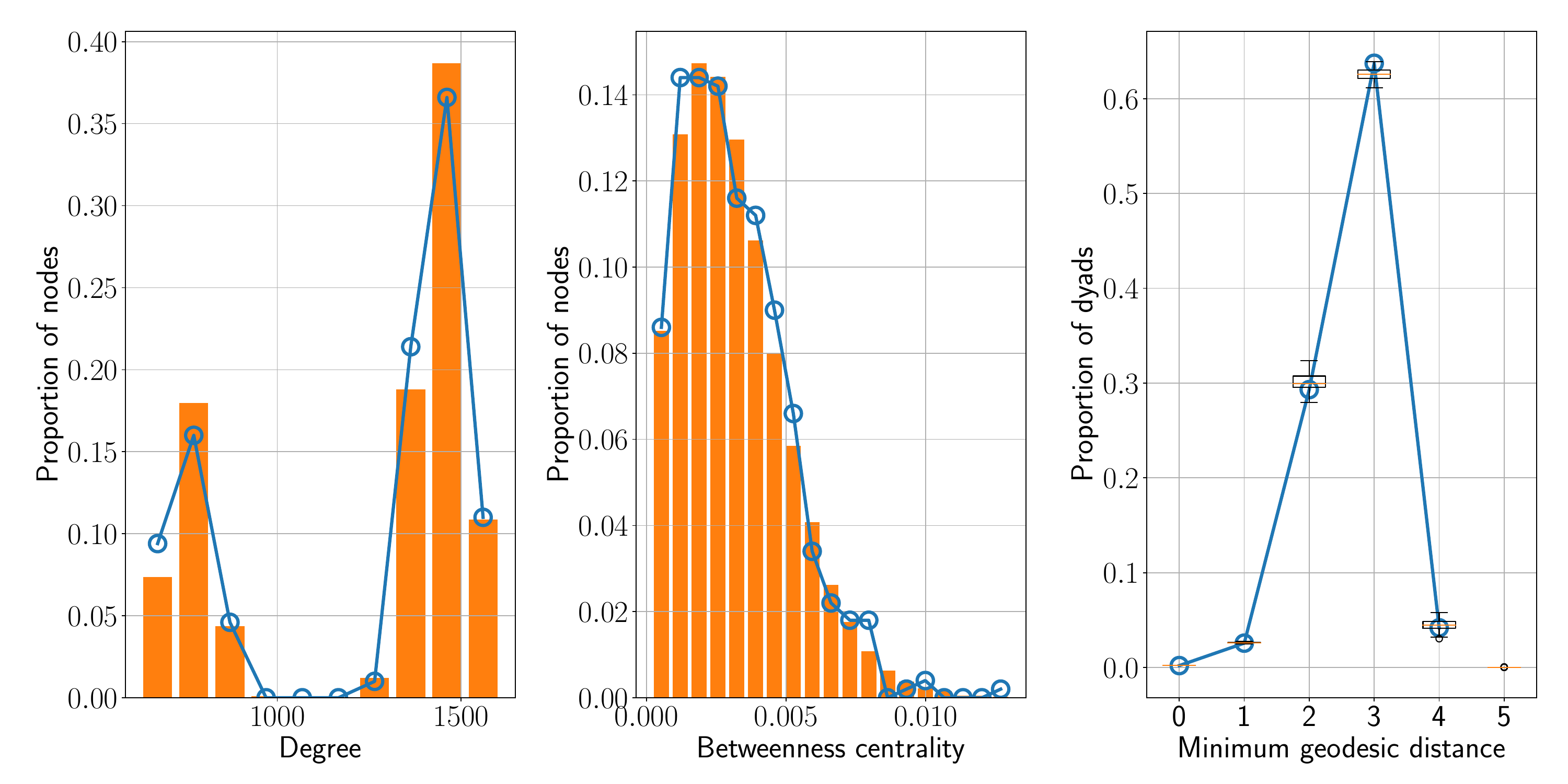}
    \caption{Comparisons between two-blocks SBMs generated from the base model (blue line) and from the discrete density estimated from latent positions (histograms and boxplots).}
    \label{fig:discrete_sbm}
\end{figure}
\end{example}

\begin{remark}
Vandermonde matrices--particularly those constructed from monomial bases--are prone to ill-conditioning, especially as the polynomial degree increases, or when the support of the distribution is large. This poses significant numerical challenges when recovering discrete distributions from moment sequences using direct matrix inversion methods, such as solving the system~\eqref{eq:vandermonde}. To mitigate this issue, we propose an alternative formulation that leverages Chebyshev polynomials of the first kind. These polynomials are orthogonal with respect to the weight function \( \frac{1}{\sqrt{1 - x^2}} \) on the interval \( [-1, 1] \), which endows them with superior numerical stability and conditioning within that domain \cite{boyd2007numerical}.

Our approach is to reformulate the system~\eqref{eq:vandermonde} on the basis of Chebyshev polynomials of the first kind. Each expression on the left-hand side of the system can be interpreted as a dot product between the unknown probability vector \( \bbp \) and a polynomial expressed on the monomial basis \( \{x^i\} \), evaluated at the support points \( \{v_r\} \). Since any polynomial of degree \( K \) can be expressed as a linear combination of the first \( K+1 \) Chebyshev polynomials, we propose to convert each monomial \( x^k \) in the system into its Chebyshev expansion. This leads to an equivalent system, with a better conditioned Vandermonde-like matrix constructed using evaluations of Chebyshev polynomials at \( \{v_r\} \).

This Chebyshev-based reformulation enhances numerical stability during reconstruction; an attractive feature when dealing with empirically estimated moments, which are inevitably affected by noise due to finite sample size. Additional methodological and implementation details are provided in Appendix~\ref{app:chebyshev_recovery}.

\end{remark}


\subsection{Continuous weights distribution}\label{sec:generation_cont}

We now move on to the case where the edge weights $W_{ij}$ are continuous random variables. That is, our goal is to determine a pdf $g_{ij}$ such that
\begin{align}
    \int_{\Omega} x^kg_{ij}(x)\, dx = \bbx_i^\top[k]\bbx_j[k],\,\, \text{ for all } k=0,1,\dots, K,
    \label{eq:moments_constraints_main}
\end{align}
where $\Omega \subset \reals$ denotes the (assumed known) support of the random variable $W_{ij}$.

To identify a unique pdf consistent with this partial information, we turn to the foundational work of Shore and Johnson  \cite{shore1980axiomatic}, who provided an axiomatic justification for using the principle of maximum entropy in such settings. Their key result shows that when one's knowledge about a probability distribution is limited to certain expectation values or constraints, the only consistent and unbiased method for selecting a distribution is to choose the one that maximizes entropy subject to those constraints. This follows from a set of axioms that any rational updating procedure should satisfy, such as consistency, uniqueness, invariance under coordinate transformations, and system independence.

In this context, entropy refers to the differential entropy of pdf $g_{ij}$, namely
\begin{align*}
    S(g_{ij}) = -\int_{\Omega}g_{ij}(x)\log g_{ij}(x)\, dx.
\end{align*}
Maximizing $ S(g_{ij})$ under the moment constraints \eqref{eq:moments_constraints_main} ensures that no additional structure or assumptions are imposed beyond what is strictly warranted by those constrains. Thus, following Shore and Johnson’s derivation, the maximum entropy principle is not simply a heuristic, but the unique method of inference consistent with the logical requirements of rational belief updating.

Therefore, we formulate the following primal optimization problem:
\begin{align}\label{eq:entropy_primal_prob}
    \max_{g \in \mathcal{F}} S(g) \,\, \text{ s. to } \int_{\Omega} x^kg(x)\, dx - m_k=0,\,\, \text{ for all } k=0,1,\dots, K,
\end{align}
where $m_k=\bbx_i^\top[k]\bbx_j[k]$ and $\mathcal{F}$ denotes the space of all probability distributions supported on $\Omega$ (as before, for clarity and to avoid notational clutter, we omit the dependence of $g$ and $m_k$ on indices $i$ and $j$). In Appendix \ref{app:maximum_entropy}, we show that the solution of \eqref{eq:entropy_primal_prob} can be obtained by solving the corresponding dual problem
\begin{align*}
    \min_{\lambda_0,\dots,\lambda_K \in \reals} \left[\sum_{k=0}^K \lambda_k m_k + \int_\Omega \exp \left( -\sum_{k=0}^K \lambda_k x^k \right) dx - m_0\right].
\end{align*}
Since this dual formulation is convex and unconstrained, it can be tackled using standard convex optimization algorithms. In practice, we solve it using the BFGS algorithm \cite{nocedal2006numerical} available through SciPy’s \texttt{minimize} function.

\begin{example}
To further illustrate the flexibility of the proposed graph generation method, we consider a two-block weighted SBM network with \( N = 500 \) nodes, comprising 350 nodes in community 1 and 150 nodes in community 2. The block probability matrix is given by  
\begin{align*}
\bbB = \left(
\begin{array}{cc}
1 & 1 \\
1 & 1
\end{array}
\right),
\end{align*}  
which implies that the graph is fully connected. This setup is chosen because assigning probabilities strictly less than one would introduce a point mass at zero in the edge-weight distribution, thereby resulting in a mixture distribution -- a setup which we address in the next section. Edge weights are sampled from distinct continuous distributions depending on the community membership of the connected nodes: for intra-community edges within community 1, weights follow a normal distribution \( \mathcal{N}(6,1) \); for intra-community edges within community 2, weights follow an exponential distribution with rate parameter \( \lambda = \frac{1}{3} \); and for inter-community edges, weights follow a normal distribution \( \mathcal{N}(1, 0.01) \). Using~\eqref{eq:sbm_normal_mean_latent_sequence}, we compute the analytical latent positions for each node, where \( m_d[k] \) now denotes the \( k \)-th moment of the corresponding distribution associated with each block. Edge weights thus arise from a mixture of three distributions, leading to a richer moment structure. 

\begin{figure}[t]
    \centering
    \includegraphics[width=\textwidth]{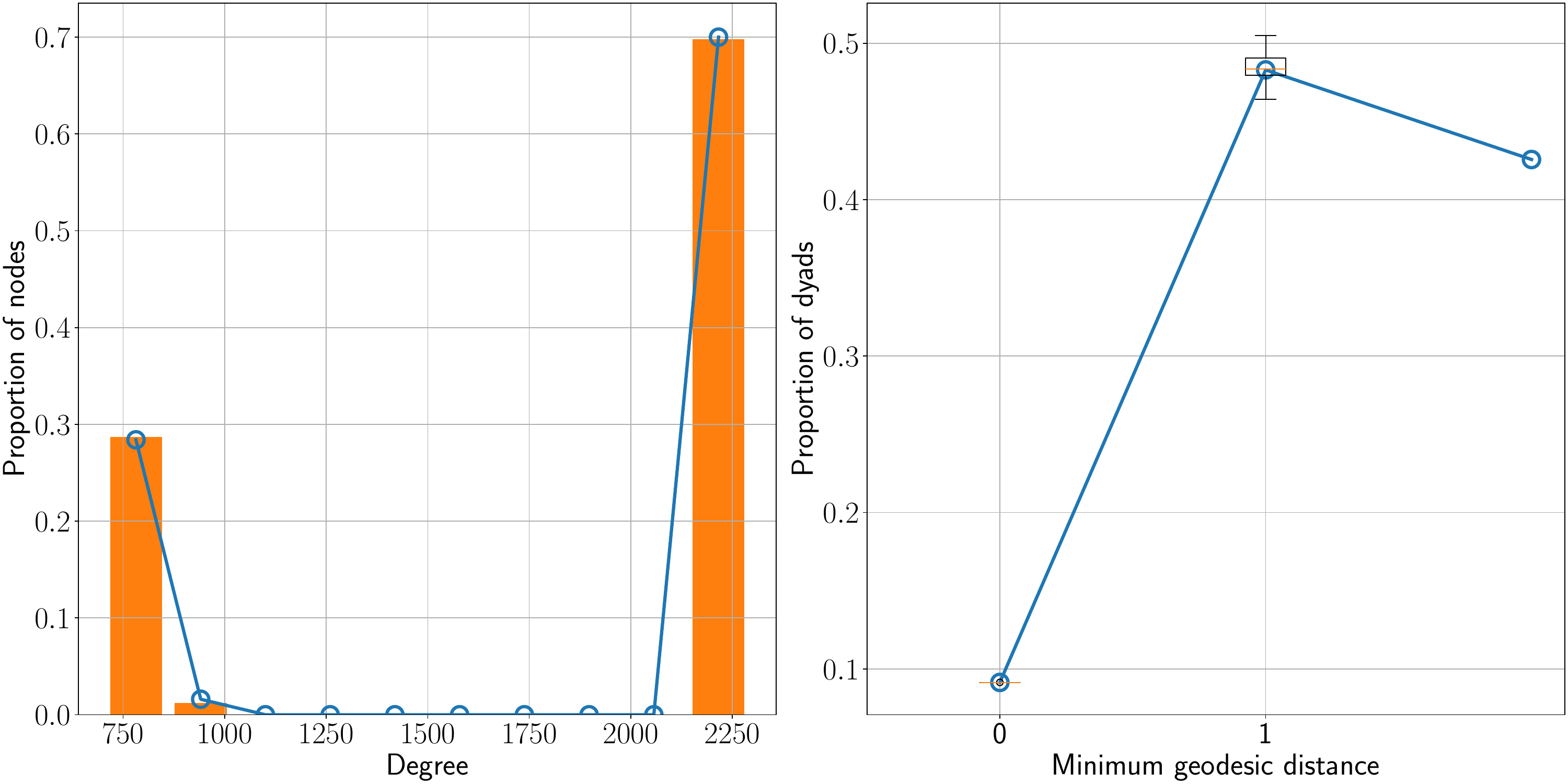}
    \caption{Comparisons between two-block SBMs generated from the base model (blue line) and from the pdf estimated with our method for solving the maximum entropy problem from latent positions (histogram and boxplot). Since in this setup graphs are fully connected, we do not report results for betweenness centrality.}
    \label{fig:continuous_mixture_sbm}
\end{figure}
\end{example}

As before, for each edge we compute the finite moment sequence \( m_{ij}[k] = \bbx_i^\top[k] \bbx_j[k] \), for \( k = 0, \dots, 5 \). We approximate the underlying edge-weight distributions using those first six moments, estimating them via the maximum entropy principle by solving the dual problem derived in Appendix~\ref{app:maximum_entropy}. Figure~\ref{fig:continuous_mixture_sbm} displays network-wide distributions of summary statistics (degree and geodesic distance) for 100 networks generated using this procedure, alongside the corresponding metrics of a single network sampled directly from the base model, i.e., a two-class, fully connected weighted SBM, and edge weights sampled according to the aforementioned mixture of continuous distributions. The metric distributions computed from the base model are in close agreement with those obtained from the generated graphs, thus validating the accuracy of the moment-based generation process in the presence of heterogeneous weight distributions.

\subsection{Mixed weights distribution}

We now focus on the case where the weights' distribution is a mixture of discrete and continuous components, with its pdf having the form
\begin{align}
    g_{ij}(x) = p_0^{ij}\delta(x) + (1-p^{ij}_0)h_{ij}(x)
    \label{eq:mixed_pdf}
\end{align}
where $\delta(x)$ is Dirac's delta, $h_{ij}(x)$ is an unknown pdf and $p_0^{ij} \in [0,1]$ is an unknown parameter that controls the edge-formation probability between nodes $i$ and $j$ . If $p_0^{ij}$ were known, we could estimate $h_{ij}$ from the prescribed moments $m_{ij}[k]=\bbx_i^\top[k]\bbx_j[k]$, since
\begin{align*}
    m_{ij}[0] &= p_0^{ij} + (1-p_0^{ij})\int_{\Omega} h_{ij}(x)\, dx, \\
    m_{ij}[k] &= (1-p_0^{ij})\int_\Omega x^k h_{ij}(x)\, dx, \quad \forall\; k\geq 1,
\end{align*}
where $\Omega$ is the support of $h_{ij}(x)$. This implies that we are looking for a pdf $h_{ij}$ with moments
\begin{align*}
    \int_{\Omega} h_{ij}(x)\, dx &= \frac{m_{ij}[0] - p_0^{ij}}{1-p_0^{ij}}, \\
    \int_\Omega x^k h_{ij}(x)\, dx &= \frac{m_{ij}[k]}{1-p_0^{ij}}, \quad \forall\; k\geq 1,
\end{align*}
so we can use the procedure in Section \ref{sec:generation_cont} to find $h_{ij}$ by maximizing its entropy.

In order to estimate $p_0^{ij}$, we propose to simultaneously estimate it for every edge via the ASE of the binary matrix $\bbA := \ind{ \bbW >\mathbf{0} }$, where $\ind{\cdot}$ denotes the entrywise matrix indicator function, i.e., $A_{ij} := \ind{W_{ij} >0}$. This estimator can be justified by noting that $\E{A_{ij}} = 1-p_0^{ij}$. Indeed,
\begin{align*}
    \E{A_{ij}} = \int_\Omega \ind{W_{ij} >0 } g(x) \, dx = \int_\Omega (1-p_0^{ij})h_{ij}(x)\, dx = 1-p_0^{ij},
\end{align*}
where we have assumed that $\Omega \subset \reals^+$. 
Let $\hbX_A$ denote the ASE of $\bbA$. Then, in light of our consistency result in Section \ref{sec:consistency}, matrix $\hbP := \hbX_A\hbX_A^\top$ converges to $\E{\bbA}$ in the $\infnorm{\cdot}$ sense, as $N\to\infty$.

\begin{example}[Reproducing real networks.] In this example the latent positions are not given and instead we estimate them from a real graph that is observed. As a result, now the input is a finite sequence of \emph{approximate} moments. We study a dataset that records football matches between national teams \cite{yang2022foot}. For a given time period, we construct a graph in which nodes represent countries, and an edge exists between two countries if they have played at least one match against each other during that period. The edge weight corresponds to the number of matches played between them. Next, we compute the latent positions for the edge weight distribution by performing the ASE of $\bbW^{(k)}$, where $\bbW^{(k)}$ denotes the $k$-th entry-wise power of the weight matrix $\bbW$, with $k$ ranging from 0 to a prespecified value $K$. Using these latent positions, we compute a finite moment sequence for each edge via the corresponding inner product and estimate its density using the procedure described earlier in this section.

By assuming a mixture-like density as in \eqref{eq:mixed_pdf}, we explicitly model the sparsity pattern of the football network. This approach enables us to generate synthetic networks that both conform to the observed sparsity pattern and exhibit edge weights that closely mimic those of the actual network. Figure \ref{fig:football_graph_generation} shows several metrics for the actual network of football matches during the period 2010--2016, alongside the corresponding metrics for 100 synthetically-generated networks using the aforementioned pdf estimation plus weighted graph sampling procedure. For this example, we used $K=2$, meaning that the zeroth, first, and second moments were employed for density estimation. 

\begin{figure}[t]
    \centering
    \includegraphics[width=\linewidth]{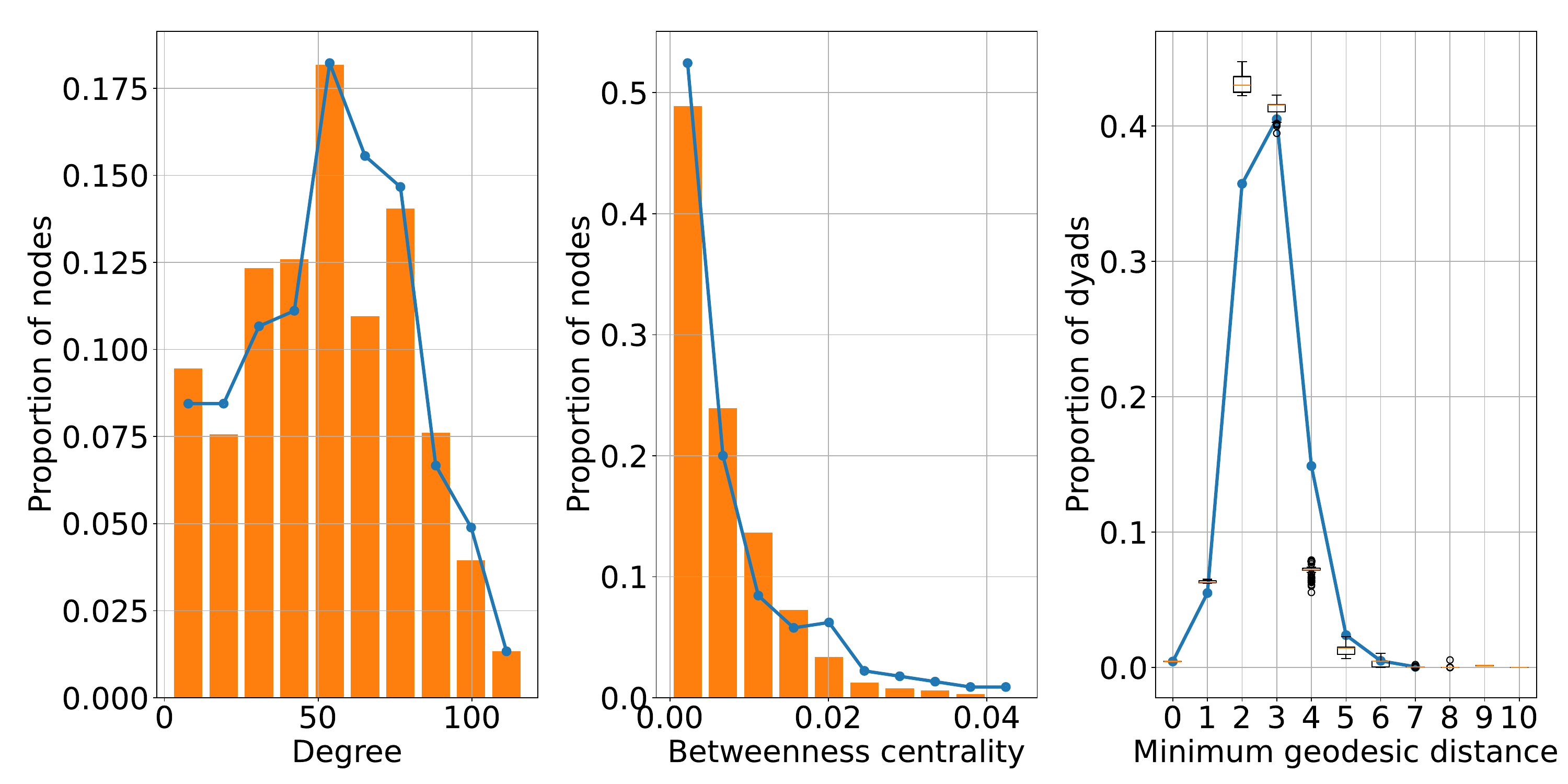}
    \caption{Graph generation metrics the football dataset \cite{yang2022foot}. Metrics for the true graph are shown with a blue solid line, while a histogram or a boxplot shows the results for the corresponding metric for 100 synthetic graphs generated using the estimated mixed densities.}
    \label{fig:football_graph_generation}
\end{figure}    

To further demonstrate the quality of the generated graphs, we conducted a community structure analysis. Since matches between countries belonging to the same football confederation are more frequent, one would expect to observe a clear community structure in the real network, with clusters roughly corresponding to confederations. This can be verified by looking at Figure \ref{fig:football_map_communities}, where we show the results of applying the classic Louvain clustering algorithm \cite{blondel2008fast} to the real graph. The algorithm detects as many clusters as there are confederations (six), with community membership reflecting, most of the times, the actual confederations teams belong to: for example, Australia is geographically in Oceania, but is a member of the Asian Football Confederation (AFC). The only major discrepancy is in North and Central America: while the countries in those subcontinents belong to the Confederation of North, Central America and Caribbean Association Football (CONCACAF), they are clustered together with southern countries, which belong to the South American Football Confederation (CONMEBOL). This might be because the former are often invited to participate in CONMEBOL's flagship tournament, the Copa Am\'erica, inflating the amount of games they play against CONMEBOL members. The fact that most Caribbean members of CONCACAF belong to a separate cluster further strenthengs this hypothesis, since they rarely play in the Copa Am\'erica.

\begin{figure}[t]
    \centering
    \includegraphics[width=\linewidth]{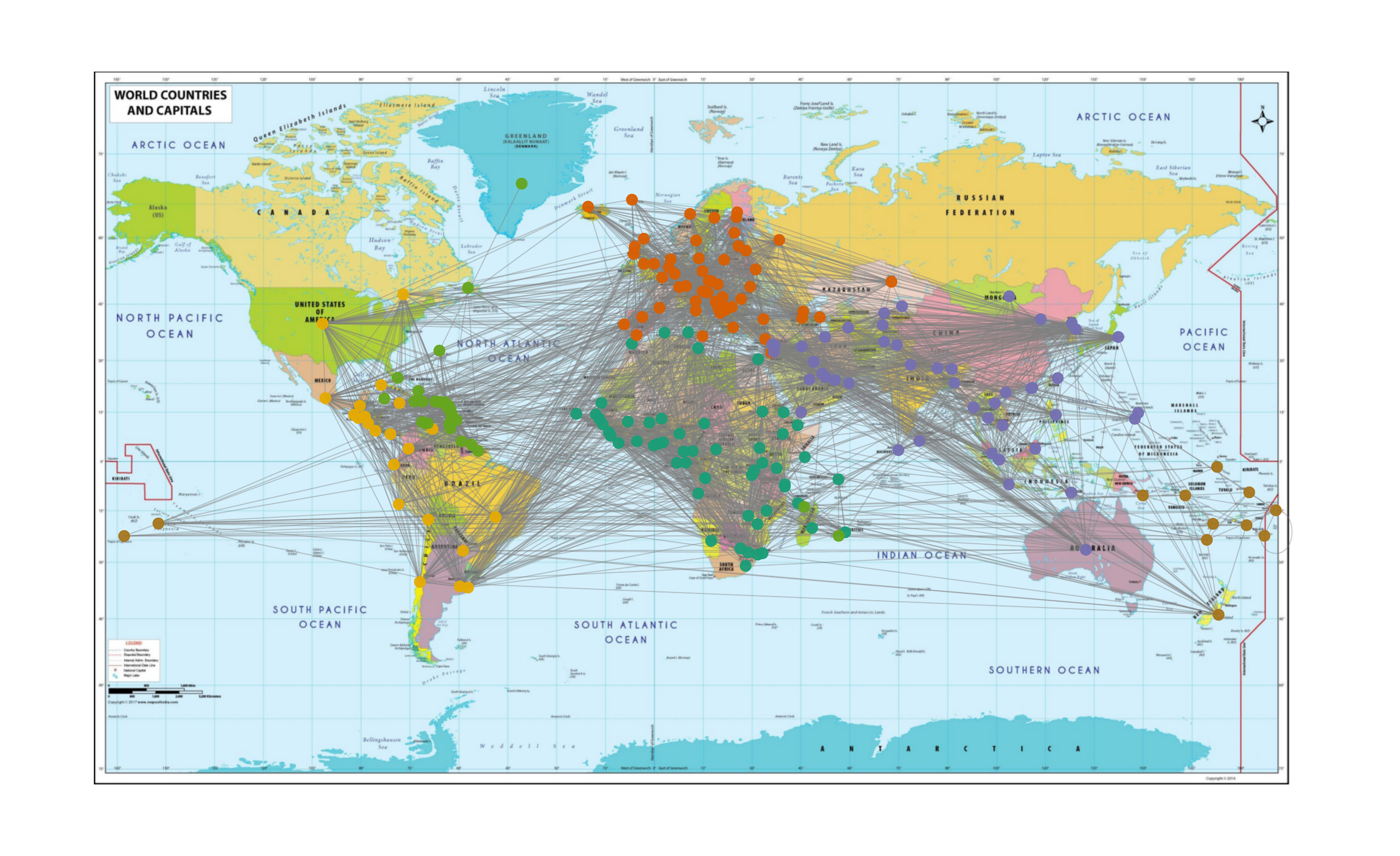}
    \caption{Result of applying the Louvain algorithm \cite{blondel2008fast} to the network of international football matches. Nodes with the same color belong to the same community.}
    \label{fig:football_map_communities}
\end{figure}  

We then ran the Louvain algorithm on the simulated networks and compared the results with those of the real network; see the results in Figure \ref{fig:football_clustering_results}. On the left panel we show a histogram of the amount of communities in the synthetic graph's largest connected component. As we can see, all graphs have either 5 or 6 communities in that component, which is consistent with the structure we discussed perviously. And while most graphs tend to have one community less than those in the real graph, we found that most of the times this is because countries from Oceania tend to be clustered with the AFC.

Besides simply looking at the amount of communities, we compared whether clusters in the simulated graphs matched those of the real network. To that end, we computed three different metrics of cluster agreement between each synthetic graph and the real one. Those metrics were V-measure \cite{rosenberg2007v}, the Adjusted Rand Index \cite{hubert1985comparing}, and the Adjusted Mutual Information \cite{vinh2009information}. On the right panel of Figure \ref{fig:football_clustering_results} we show a boxplot for each metric. Apparently, the overall community structure in the synthetic graphs follows closely that of the actual network.

\begin{figure}[t]
    \centering
    \includegraphics[width=0.85\linewidth]{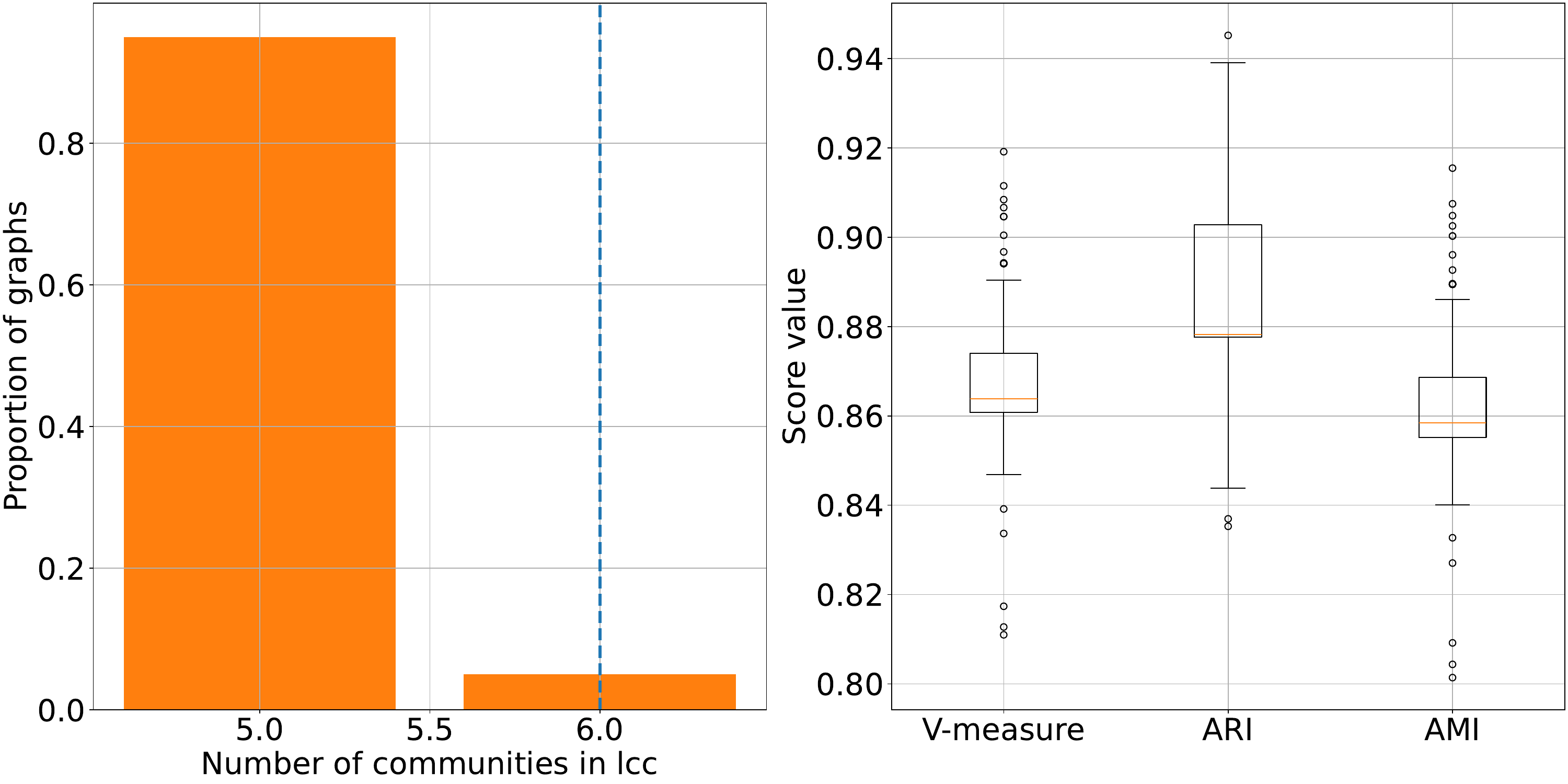}
    \caption{Comparisons between community structure of the real football matches network and its synthetic replicates. Left: histogram of number of communities in the largest connected component (lcc) for synthetic graphs. Right: Boxplot for three metrics of clustering agreement between real and synthetic networks: V-measure, Adjusted Rand Index (ARI) and Adjusted Mutual Information (AMI). }
    \label{fig:football_clustering_results}
\end{figure}  

\end{example}

\begin{remark}
Since the estimation of $p_0^{ij}$ is based on the consistency result from Section \ref{sec:consistency}, which guarantees convergence as the number of nodes $N\to\infty$, the finite size of the graph introduces a non-negligible finite-sample deviation. This deviation partially accounts for the (arguably quite minor) discrepancies observed in the metrics shown in Figure \ref{fig:football_graph_generation}, underscoring the method's robustness.
\end{remark}

\begin{remark}
The entropy-based graph regeneration procedure described in this section provides a way to reproduce weighted networks that match a finite set of estimated moments. This naturally suggests the possibility of using regenerated graphs to approximate the distribution of network summary statistics, such as subgraph counts, degree-based measures, or spectral functionals, in a bootstrap-like fashion. Related bootstrap methodologies have been developed recently for latent-space network models in the context of the vanilla (i.e., unweighted) RDPG model. In particular, the bootstrap framework put forth in~\cite{levin2025bootstrapping} establishes consistency results for U-statistic-type graph functionals and for resampling entire networks under an RDPG generative assumption. These rely on consistent estimation of the latent positions and on representations of many network statistics as functionals of these positions.
    
While the elegant ideas in~\cite{levin2025bootstrapping} provide a valuable conceptual foundation, extending such bootstrap guarantees to the present WRDPG framework is nontrivial. Here, regenerated networks are obtained via an entropy maximization principle using only finitely many moments, and the latent structure itself may vary across moments. A rigorous bootstrap theory would therefore need to account for both moment truncation effects and the distributional properties induced by the entropy-based graph regeneration scheme. We regard such developments as a meaningful direction for future research.
\end{remark}

\section{Conclusions}

We developed the Weighted Random Dot Product Graph (WRDPG) model, which extends the vanilla Random Dot Product Graph (RDPG) framework to effectively capture the intricacies of weighted graphs characterized by heterogeneous edge weight distributions. Our approach assigns latent positions to nodes, parameterizing edge weight distributions' moments using moment-generating functions. This innovative modeling strategy enhances the WRDPG's capability to distinguish between weight distributions that may have the same mean but differ in higher-order moments, thus providing a richer representation of network data than previous attempts to weighted RDPG modeling.
We derived statistical guarantees (consistency and asymptotic normality) for our estimator of latent positions, leveraging the well-established adjacency spectral embedding method. While for simplicity we assume a common embedding dimension across all moment matrices, the WRDPG framework can be extended to allow moment-specific dimensions, and we expect robustness to dimension overspecification similar to that observed in the vanilla RDPG~\cite{taing2026effect}.

Furthermore, we developed a generative framework that approximates desired weight distributions from a given (or estimated) latent position sequence and facilitates sampling synthetic graphs that closely mimic the structure and characteristics of real-world networks. Through a series of illustrative examples, we have demonstrated the effectiveness of the WRDPG model in practical applications. We show how our model can recover latent position structures and replicate key graph metrics, exhibiting a discriminative power essential for community detection tasks. We have also shown that the model accommodates various edge weight distributions, ranging from discrete to continuous, or even mixed random variables. To achieve this versatility, we developed an improved computational method to estimate a continuous density by maximizing the entropy subject to moment constraints. This result is of independent interest and could permeate benefits in other, broader contexts. In light of this, a natural direction for future research is the development of bootstrap-based inference methods for the WRDPG model. Establishing rigorous guarantees in this setting, however, would require controlling both moment truncation effects and the distributional features induced by the entropy maximization scheme, extending recent bootstrap results available for the vanilla RDPG framework~\cite{levin2025bootstrapping}. The findings of this study highlight the importance of incorporating higher-order moments in graph models, contributing to the methods of inference and generation in network data analysis. All techniques developed in this paper can be extended to heterophilous graphs, i.e., graphs where the moment matrices are not necessarily positive semidefinite, by bringing to bear the generalized RDPG model introduced in \cite{rubin2022statistical}. Asymptotic results remain tractable in this setting, since our methodology is based on that of \cite{gallagher2023spectral}, which explicitly accounts for heterophilic graphs in their model. Future work may explore further adaptations of the WRDPG model, including enhancements in (real-time) computational efficiency and applications to evolving or dynamic networks. An additional promising avenue concerns the development of clustering guarantees in settings where community structure is identifiable only through higher-order moments. Such results would require new identifiability and separation conditions across moments, as well as a careful analysis of the trade-offs inherent in leveraging higher-order information. A related open problem is the design of principled methods for aggregating information across moment matrices, such as joint embedding approaches, together with theoretical guarantees tailored to the WRDPG setting, where latent structures may vary across moments.

\begin{appendix}
\section{Proof of Proposition \ref{prop:remainders}}\label{app:technical_lemmas}
Here we state and prove some auxiliary results that will be necessary in order to prove Proposition \ref{prop:remainders}.
We begin with a lemma that establishes a bound on the MGF of a sub-Weibull random variable near the origin.

\begin{lemma}\label{lemma:mgf_weibull}
    Let \( Y \) be a sub-Weibull random variable with parameter \( \theta > 1 \); that is, there exists a constant \( C > 0 \) such that
    \begin{displaymath}
        \P{|Y| \geq t} \leq 2 \exp\left( - \left( \frac{t}{C} \right)^{1/\theta} \right) \quad \text{for all } t \geq 0.
    \end{displaymath}
   
    Then there exist $T>1$ and \( C_1> 0 \), depending only on \( \theta \) and \( C \), such that:
    \begin{displaymath}
        \E{e^{\lambda |Y|}} \leq \exp\left( - C_1 \lambda^{1/(1 - \theta)} \right) \text{ for all } \lambda \in [0,\lambda_0],
    \end{displaymath}
    where
    \begin{displaymath}
        \lambda_0 = \frac{1}{\theta C^{1/\theta}} T^{(1 - \theta)/\theta}.
    \end{displaymath}
\end{lemma}
\begin{proof}
    By Fubini’s theorem, we can write:
    \begin{displaymath}
        \E{e^{\lambda |Y|}} = \int_0^\infty \lambda e^{\lambda t} \P{|Y| \geq t} dt.
    \end{displaymath}
    Using the sub-Weibull tail bound, we obtain:
    \begin{displaymath}
        \E{e^{\lambda |Y|}} \leq 2\lambda \int_0^\infty \exp\left( \lambda t - \left( \frac{t}{C} \right)^{1/\theta} \right) dt.
    \end{displaymath}
    Let \(\displaystyle \phi_\lambda(t) = \lambda t - \left( \frac{t}{C} \right)^{1/\theta} \), and define:
    \begin{displaymath}
        I(\lambda) = \int_0^\infty \exp\left( \phi_\lambda(t) \right) dt.
    \end{displaymath}
    The function \( \phi_\lambda(t) \) attains its unique maximum at
    \begin{displaymath}
        t_\lambda = \left( \lambda \theta C^{1/\theta} \right)^{\theta/(1 - \theta)}.
    \end{displaymath}
    To apply Laplace's method (see \cite[Chapter 4]{de2014asymptotic}), we require \( t_\lambda \geq T \) for some large enough \( T > 1 \), so that the integrand is sharply peaked near \( t_\lambda \). Solving \( t_\lambda \geq T \) yields:
    \begin{displaymath}
        \lambda \leq \lambda_0 = \frac{1}{\theta C^{1/\theta}} T^{(1 - \theta)/\theta}.
    \end{displaymath}
    Under this condition, Laplace’s method gives:
    \begin{displaymath}
        I(\lambda) \leq K \exp\left( \phi_\lambda(t_\lambda) \right),
    \end{displaymath}
    for some constant \( K > 0 \) depending only on \( \theta \) and \( C \). Evaluating \( \phi_\lambda(t_\lambda) \), we find:
    \begin{align*}
        \phi_\lambda(t_\lambda) = \left( \frac{1}{\theta} - 1 \right) (\theta C)^{1/(1 - \theta)} \lambda^{1/(1 - \theta)}.
    \end{align*}
    Since \( \theta > 1 \), we define
    \begin{displaymath}
        C_1 = -\left( \frac{1}{\theta} - 1 \right) (\theta C)^{1/(1 - \theta)} > 0,
    \end{displaymath}
    so that
    \begin{displaymath}
        \phi_\lambda(t_\lambda) = - C_1 \lambda^{1/(1 - \theta)}.
    \end{displaymath}
    Putting everything together,
    \begin{displaymath}
        \E{e^{\lambda |Y|}} \leq 2\lambda I(\lambda) \leq 2\lambda K \exp\left( - C_1 \lambda^{1/(1 - \theta)} \right).
    \end{displaymath}
    Since $\lambda_0 \to 0$ as $T\to \infty$, we can choose a $T$ large enough so that $2\lambda_0 K \leq 1$. Therefore, for all $\lambda\in[0,\lambda_0]$ we have:
    \begin{displaymath}
        \E{e^{\lambda |Y|}} \leq 2\lambda_0 K\exp\left( - C_1 \lambda^{1/(1 - \theta)} \right) \leq \exp\left( - C_1 \lambda^{1/(1 - \theta)} \right),
    \end{displaymath}
    as claimed.
\end{proof}

We next present a concentration result for the sum of independent, centered sub-Weibull rvs with a common parameter $\theta$. The proof follows standard arguments used in analogous concentration results for sub-Gaussian and sub-exponential rvs (see Theorems 2.6.3 and 2.8.1 in \cite{vershynin2018hdpbook}).

\begin{proposition}\label{prop:concentration_sub_weibull}
    Let $Y_1,\dots,Y_N$ be independent, mean-zero, sub-Weibull random variables with parameter $\theta > 1$. That is, for each $i = 1, \dots, N$, there exists a constant \( C_i > 0 \) such that
    \begin{displaymath}
        \P{|Y_i| \geq t} \leq 2 \exp\left( - \left( \frac{t}{C_i} \right)^{1/\theta} \right) \quad \text{for all } t \geq 0.
    \end{displaymath}
    For any $T > 0$, define $\lambda_{\min} = \frac{1}{\theta C_{\max}^{1/\theta}} T^{(1 - \theta)/\theta}$, where $\ds C_{\max} = \max_{i=1,\dots,N} C_i$.

    Then there exists a large enough $T$ such that for all $t \geq 0$, it holds that
    \begin{displaymath}
        \P{\left\lvert \sum_{i=1}^N Y_i \right\rvert \geq t} \leq 2 \expon{ - \min \left\lbrace \lambda_{\min} t + NK_1 \lambda_{\min}^{1/(1-\theta)} , \left( \frac{t}{K_2} \right)^{1/\theta} N^{(\theta - 1)/\theta} \right\rbrace },
    \end{displaymath}
    where $K_1, K_2 > 0$ are constants depending only on $\theta$.
\end{proposition}

\begin{proof}
Let \( S := \sum_{i=1}^N Y_i \). By Markov’s inequality and independence:
\begin{displaymath}
    \P{S \geq t} = \P{e^{\lambda S} \geq e^{\lambda t}} \leq e^{-\lambda t} \prod_{i=1}^N \E{e^{\lambda Y_i}}.
\end{displaymath}

Now, using Lemma~\ref{lemma:mgf_weibull}, there exists a large enough $T > 1$ such that for every $i = 1,\dots,N$ there exists a constant $C_{1,i}> 0$ (depending only on $C_i$ and $\theta$) such that for all $\lambda \in [0, \lambda_{\min}]$:
\begin{displaymath}
    \E{e^{\lambda |Y_i|}} \leq \exp\left( - C_{1,i} \lambda^{1/(1 - \theta)} \right).
\end{displaymath}
Using the fact that $Y_i$ is mean-zero, we have:
\begin{displaymath}
    \E{e^{\lambda Y_i}} \leq \E{e^{\lambda |Y_i|}} \leq \exp\left( - C_{1,i} \lambda^{1/(1 - \theta)} \right).
\end{displaymath}
Hence,
\begin{displaymath}
    \P{S \geq t} \leq \exp\left( - \lambda t - \sum_{i=1}^N C_{1,i} \lambda^{1/(1 - \theta)}  \right).
\end{displaymath}

Let $K_{1} = \min_i C_{1,i}$. Then $K_1$ depends only on $ \theta $, since the $ C_{1,i} $'s depend only on $ \theta $. Therefore:
\begin{align}
    \P{S \geq t} \leq \expon{ - \lambda t - N K_1 \lambda^{1/(1 - \theta)} }.
    \label{eq:concentration_form}
\end{align}
Minimizing the exponent over $ \lambda \in [0, \lambda_{\min}] $, the optimal choice is:
\begin{displaymath}
    \lambda^* = \min\left\lbrace \lambda_{\min}, \left( \frac{t (\theta-1)}{ N K_1} \right)^{(1 - \theta)/\theta} \right\rbrace.
\end{displaymath}
Substituting into \eqref{eq:concentration_form}, we obtain:
\begin{displaymath}
    \P{S \geq t} \leq  \exp\left( - \min \left\lbrace \lambda_{\min} t + NK_1\lambda_{\min}^{1/(1 - \theta)}, \left( \frac{t}{K_2} \right)^{1/\theta} N^{(\theta - 1)/\theta} \right\rbrace \right),
\end{displaymath}
where $ K_2= \frac{\theta^\theta}{(\theta-1)^{\theta-1}}(K_1^{(\theta-1)/\theta}) $ depends only on $ \theta $.

A symmetric argument for $-S$ gives the same bound for $ \P{-S \geq t} $, and the result follows.
\end{proof}

We next state a lemma that establishes a bound on the difference between the matrix of entry-wise self-products, $\bbW^{(k)}$, and the true moments matrix, $\bbX[k]\bbX^\top[k]$. The proof follows the general scheme of \cite[Proposition 1]{gallagher2023spectral}, but differs in a key aspect: we rely on a different result for tail bounds of sums of independent matrices. This is due to the fact that the concentration result they invoke does not hold in our more general setup.

\begin{lemma}\label{lemma:A-P_spectral_norm}
Let $(\bbW,\bbX_k) \sim \mathrm{WRDPG}(F)$ satisfy the hypotheses of Theorem \ref{theorem:consistency}. For each fixed integer $k\geq 0$, let $\bbW_k := \bbW^{(k)}$ and $\bbM_k := \bbX[k]\bbX^\top[k]$. Then 
\begin{align*}
\spectral{\bbW_k - \bbM_k} = \OP{\log^{k\theta}N}.
\end{align*}
\end{lemma}

\begin{proof}
Note that the definition of our model implies that $\E{\bbW_k} = \bbM_k$, so we are trying to control the spectral norm of the centered matrix $\bbW_k-\bbM_k$. To that end, we will use a tail bound for sums of independent matrices \cite[Corollary 3.7]{tropp2012user}. That result can be stated as follows: assume we have a finite sequence $\{\bbY_l\}_{l=1}^R$ of independent, self-adjoint, random $N\times N$ matrices that satisfy
\begin{displaymath}
    \E{e^{\lambda \bbY_l}} \preccurlyeq e^{g(\lambda)\bbA_l} \,\,\,\text{ for all } \lambda \in [0,\lambda_0]
\end{displaymath}
for some function $g:[0,\lambda_0] \to \reals$ and a finite sequence $\{\bbA_l\}_{l=1}^{R}$ of fixed self-adjoint matrices. Here, $\succcurlyeq$ denotes the semidefinite ordering on Hermitian matrices, i.e., $\bbA \succcurlyeq \bbB$ if $\bbA-\bbB$ is positive semi-definite. Define the scale parameter 
\begin{displaymath}
    \rho = \spectral{ \sum_{l=1}^R \bbA_l } .
\end{displaymath}
Then for all $t\geq 0$ it holds:
\begin{displaymath}
    \P{\spectral{ \sum_{l=1}^R \bbY_l } \geq t} \leq N \inf_{\lambda \in [0,\lambda_0]} \expon{-\lambda t + g(\lambda)\rho}.
\end{displaymath}

In order to use this result, for each $1\leq j\leq i\leq N$ let $\bbY_{ij}$ be the $N\times N$ matrix whose $(i,j)$ and $(j,i)$ entries equal $W_{ij}^k - \bbx_{i}^\top[k]\bbx_j[k] := Y_{ij}$, with the rest being 0. Then, $\sum_{ij} \bbY_{ij} = \bbW_k-\bbM_k$ and $\E{\bbY_{ij}}=\mathbf{0}$ since $\E{Y_{ij}} = 0$. Also note that if $\bbe_i$ is the $i$-th vector of the canonical basis of $\reals^N$, we have:
\begin{align*}
    \bbY_{ij} = Y_{ij}(\bbe_i\bbe_j^\top+\bbe_j\bbe_i^\top) \Rightarrow \bbY_{ij}^p =
    \begin{cases}
        Y_{ij}^p(\bbe_i\bbe_j^\top+\bbe_j\bbe_i^\top) \text { if } p \text{ is odd}\\
        Y_{ij}^p(\bbe_i\bbe_i^\top+\bbe_j\bbe_j^\top) \text { if } p \text{ is even, } p\geq 2
    \end{cases}.
\end{align*}
Therefore:
\begin{align*}
     e^{\lambda \bbY_{ij}} &= \bbI + \sum_{p \text { odd}} \frac{(\lambda Y_{ij})^p}{p!} (\bbe_i\bbe_j^\top+\bbe_j\bbe_i^\top) + \sum_{p \text { even},\, p\geq 2} \frac{(\lambda Y_{ij})^p}{p!} (\bbe_i\bbe_i^\top+\bbe_j\bbe_j^\top)\\
     &= \bbI + \sinh(\lambda Y_{ij}) (\bbe_i\bbe_j^\top+\bbe_j\bbe_i^\top) + (\cosh(\lambda Y_{ij})-1) (\bbe_i\bbe_i^\top+\bbe_j\bbe_j^\top).
\end{align*}
This means that $e^{\lambda \bbY_{ij}}$ has ones along the diagonal, except at the $(i,i)$ and $(j,j)$ entries, which are equal to $\cosh(\lambda Y_{ij})$. Its off-diagonal entries are zero, except for the $(i,j)$ and $(j,i)$ entries, which are equal to $\sinh(\lambda Y_{ij})$. Using the identity $e^{|x|} - \cosh(x) = |\sinh(x)|$ for all $x \in \reals$, together with the Gershgorin circle theorem, we obtain:
\begin{displaymath}
    e^{\lambda \bbY_{ij}} \preccurlyeq \bbI + (e^{|\lambda Y_{ij}|}-1)(\bbe_i\bbe_i^\top+\bbe_j\bbe_j^\top) = \expon{|\lambda Y_{ij}|(\bbe_i\bbe_i^\top+\bbe_j\bbe_j^\top)},
\end{displaymath}
where the last equality follows from the fact that the matrix $\bbI + \left(e^{|\lambda Y_{ij}|} - 1\right)(\bbe_i\bbe_i^\top + \bbe_j\bbe_j^\top)$ is diagonal, with all entries equal to 1 except at the $(i,i)$ and $(j,j)$ positions, which are equal to $e^{|\lambda Y_{ij}|}$. 
Since the matrix expectation preserves the semidefinite order, this in turn implies
\begin{displaymath}
    \E{e^{\lambda \bbY_{ij}}} \preccurlyeq \E{\expon{|\lambda Y_{ij}|(\bbe_i\bbe_i^\top+\bbe_j\bbe_j^\top)}}.
\end{displaymath}
As discussed in Remark~\ref{remark:weibull}, Assumption~\ref{assumption:bounded_rv} implies that $Y_{ij}$ is a sub-Weibull random variable with parameter $k\theta$. Therefore, by Lemma~\ref{lemma:mgf_weibull}, there exists a constant $C_{1}^{i,j}$ and a threshold $\lambda_0^{i,j} > 0$ such that for all $\lambda \in [0, \lambda_0^{i,j}]$,
\begin{displaymath}
     \E{e^{\lambda \bbY_{ij}}} \preccurlyeq \E{\expon{-C_{1}^{i,j} \lambda^{1/(1-k\theta)} (\bbe_i\bbe_i^\top+\bbe_j\bbe_j^\top)}}.
\end{displaymath}
Letting $K_1 = \min_{i,j} C_{1}^{i,j}$ and $\lambda_{\min} = \min_{i,j} \lambda_0^{i,j}$, we obtain that for all $\lambda \in [0,\lambda_{\min}]$
\begin{displaymath}
     \E{e^{\lambda \bbY_{ij}}} \preccurlyeq \E{\expon{-K_{1} \lambda^{1/(1-k\theta)}(\bbe_i\bbe_i^\top+\bbe_j\bbe_j^\top)}}.
\end{displaymath}
Thus, we can apply the tail bound for sums of independent random matrices stated above, with $g(\lambda) = -K_1 \lambda^{1/(1 - k\theta)}$ and $\bbA_{ij} = \bbe_i\bbe_i^\top + \bbe_j\bbe_j^\top$. Therefore,
\begin{align}
    \P{\spectral{\bbW_k - \bbM_k} \geq t} \leq N \inf_{\lambda \in [0,\lambda_{\min}]} \expon{-\lambda t -K_1\lambda^{1/(1-k\theta)}\rho},
    \label{eq:inf_Wk-Mk}
\end{align}
where
\begin{displaymath}
    \rho = \spectral{\sum_{1\leq i \leq j\leq N} (\bbe_i\bbe_i^\top+\bbe_j\bbe_j^\top)} = \spectral{N\bbI} = N.
\end{displaymath}
Note that the function inside the infimum in~\eqref{eq:inf_Wk-Mk} is equivalent to that on the right-hand side of~\eqref{eq:concentration_form}. Therefore, an argument analogous to the one in the proof of Proposition~\ref{prop:concentration_sub_weibull} yields
\begin{align*}
    \P{\spectral{\bbW_k - \bbM_k} \geq t} \leq N \exp & \left( - \min \left\lbrace \lambda_{\min} t + NK_1\lambda_{\min}^{1/(1 - k\theta)}, \left( \frac{t}{K_2} \right)^{1/(k\theta)} N^{(k\theta - 1)/(k\theta)} \right\rbrace \right),
\end{align*}
where $K_2= \frac{(k\theta)^{k\theta}}{(k\theta-1)^{k\theta-1}}(K_1^{(k\theta-1)/(k\theta)})$. Since for large $t$ the term involving $t^{1/(k\theta)}$ dominates the minimum, we have that, for large enough $t$, 
\begin{align*}
    \P{\spectral{\bbW_k - \bbM_k} \geq t} \leq N \expon{ - \left( \frac{t}{K_2} \right)^{1/(k\theta)} N^{(k\theta - 1)/(k\theta)}}.
\end{align*}
Choosing $t = K_2\log^{k\theta} N$ then yields the desired result.
\end{proof}

The following propositions can be proven with the same arguments found in  \cite[Proposition 16, Lemma 17]{lyzinski2017community} and \cite[Propositions 3 and 5]{gallagher2023spectral}, respectively. While we omit the detailed proofs, we provide a brief outline of the key ideas underlying each result.

Proposition~\ref{prop:procrustes} follows from a Procrustes-alignment style argument, which leverages the relationship between the principal angles of the subspaces spanned by the columns of \(\bbU_k\) and \(\hbU_k\), and the eigenvalues of the difference \(\hbU_k\hbU_k^\top - \bbU_k\bbU_k^\top\), together with an application of the Davis–Kahan theorem \cite{yu2015useful}. Proposition~\ref{prop:vince_2} can be proved using similar techniques.

Proposition~\ref{prop:vince_1} is derived via a Hoeffding-type concentration argument analogous to that used in the proof of Theorem~\ref{theorem:consistency}. Lastly, Proposition~\ref{prop:gallagher_1} is a direct consequence of Propositions~\ref{prop:procrustes} and~\ref{prop:vince_2}.

\begin{proposition} \label{prop:procrustes}
For each $k$, let
\begin{align*}
    \bbU_k^\top\hbU_k =  \bbS_{k_1}\bbSigma_k\bbS_{k_2}^\top
\end{align*}
be the singular value decomposition of $\bbU_k^\top\hbU_k$. If $\bbS_k := \bbS_{k_1}\bbS_{k_2}^\top$, then it holds that
    \begin{align*}
        \Fro{\bbU_k^\top \hbU_k-\bbS_k}= \OP{N^{-1} \log^{k\theta}{N}}.
    \end{align*}
\end{proposition}

\begin{proposition}\label{prop:vince_1}
    For each $k$, the following holds:
    \begin{align*}
        &\Fro{\bbU_k^\top(\bbW_k-\bbM_k)\bbU_k} = \OP{\log^{k\theta}{N}}.
    \end{align*}
\end{proposition}

\begin{proposition}\label{prop:vince_2}
    For each $k$, the following holds:
    \begin{align*}
        &\Fro{\hbU_k\hbU_k^\top - \bbU_k\bbU_k^\top} = \OP{N^{-1}\log^{k\theta}{N}}\\
        &\Fro{\hbU_k - \bbU_k\bbU_k^\top\hbU_k}=\OP{N^{-1}\log^{k\theta}{N}}
    \end{align*}
\end{proposition}

\begin{proposition}\label{prop:gallagher_1}
For each $k$, let $\bbS_k$ be as in Proposition \ref{prop:procrustes}. Then the following holds:
    \begin{align*}
       \Fro{\bbS_k\hbD_k-\bbD_k\bbS_k} &= \OP{\log^{k\theta}{N}},\\
        \Fro{\bbS_k\hbD_k^{1/2}-\bbD_k^{1/2}\bbS_k} &= \OP{N^{-1/2}\log^{k\theta}{N}}\\
        \Fro{\bbS_k\hbD_k^{-1/2}-\bbD_k^{-1/2}\bbS_k} &= \OP{N^{-3/2}\log^{k\theta}{N}}.
    \end{align*}
\end{proposition}

We are now ready to prove Proposition \ref{prop:remainders}. We remind the reader its statement:
\begin{restatedproposition}{prop:remainders}
For each fixed integer $k\geq 0$, let $\bbW_k := \bbW^{(k)}$. Also let $\bbM_k := \bbX[k]\bbX^\top[k]$ and denote its spectral decomposition as $\bbM_k = \bbU_k\bbD_k\bbU_k^\top$. Let $\bbS_k \in O(d)$ be as in Proposition \ref{prop:gallagher_1} and  define $\bbR_{k_1}$, $\bbR_{k_2}$, $\bbR_{k_3}$ and $\bbR_{k_4}$ as
    \begin{align*}
        \bbR_{k_1} &= \bbU_k\left(\bbU_k^\top \hbU_k\hbD_k^{1/2} -\bbD_k^{1/2} \bbS_k  \right),\\
        \bbR_{k_2} &=\left(\bbI- \bbU_k\bbU_k^\top \right)\left(\bbW_k- \bbM_k \right)\left(\hbU_k -\bbU_k\bbS_k  \right)\hbD_k^{-1/2},\\
        \bbR_{k_3} &=- \bbU_k\bbU_k^\top\left(\bbW_k- \bbM_k \right)\bbU_k\bbS_k  \hbD_k^{-1/2},\\
        \bbR_{k_4} &=\left(\bbW_k- \bbM_k \right)\bbU_k\left(\bbS_k\hbD_k^{-1/2} -\bbD_k^{-1/2}\bbS_k  \right).
    \end{align*}
     Then the following holds:
     \begin{align*}
        \twoToInf{\bbR_{k_1}} &= \OP{N^{-1} \log^{k\theta} N}, \\
        \twoToInf{\bbR_{k_2}} &= \OP{N^{-3/2} \log^{2k\theta}  N},\\
        \twoToInf{\bbR_{k_3}} &= \OP{N^{-1} \log^{k\theta} N},\\
        \twoToInf{\bbR_{k_4}} &= \OP{N^{-3/2} \log^{2k\theta} N}.
    \end{align*}
\end{restatedproposition}

\begin{proof}
    In order to bound $\bbR_{k_1}=\bbU_k\left(\bbU_k^\top \hbU_k\hbD_k^{1/2} -\bbD_k^{1/2} \bbS_k  \right)$, note that:
    \begin{align*}
        \twoToInf{\bbR_{k_1}} &\leq \twoToInf{\bbU_k}  \spectral{\bbU_k^\top \hbU_k\hbD_k^{1/2} -\bbD_k^{1/2} \bbS_k  }\\
         & \leq  \twoToInf{\bbU_k}  \left(\Fro{\left(\bbU_k^\top \hbU_k-\bbS_k\right)\hbD_k^{1/2}}+\Fro{ \bbS_k\hbD_k^{1/2} - \bbD_k^{1/2} \bbS_k}\right).
    \end{align*}
    Since $\bbU_k\bbD_k^{1/2} = \bbX[k]\bbQ_k$ with $\bbQ_k \in O(d)$, we have that
    \begin{align*}
    \twoToInf{\bbU_k} \leq \twoToInf{\bbX[k]}\spectral{\bbD_k^{-1/2}}.
    \end{align*}
    Since the inner product between the rows of $\bbX[k]$ equals the $k$-th moment of some random variable--which, by the definition of our model, we assume to be finite for all $k$--this implies that all rows of $\bbX[k]$ (which lie in $\reals^d$) must have bounded norm. Therefore, $\twoToInf{\bbX[k]}$ is finite and does not depend on $N$. In Lemma~\ref{lemma:sussman_Mk} of Appendix~\ref{app:eigenvalues}, we show that Assumption~\ref{assumption:full_Rank} implies all nonzero eigenvalues of $\bbM_k$ are $\EquivP{N}$, so $\spectral{\bbD_k^{-1/2}} = \OP{N^{-1/2}}$. Consequently, $\twoToInf{\bbU_k} = \OP{N^{-1/2}}$.%
    Combining this with Propositions \ref{prop:procrustes} and \ref{prop:gallagher_1} yields the desired bound for $\bbR_{k_1}$.

    In a similar vein, for $\bbR_{k_3} = - \bbU_k\bbU_k^\top\left(\bbW_k- \bbM_k \right)\bbU_k\bbS_k  \hbD_k^{-1/2}$ it holds:
    \begin{align*}
        \twoToInf{\bbR_{k_3}} &\leq \twoToInf{\bbU_k} \spectral{\bbU_k^\top\left(\bbW_k- \bbM_k \right)\bbU_k\bbS_k  \hbD_k^{-1/2}}\\
         &\leq \twoToInf{\bbU_k} \Fro{\bbU_k^\top\left(\bbW_k- \bbM_k \right)\bbU_k} \spectral{\hbD_k^{-1/2}}.
    \end{align*}

    Using Proposition \ref{prop:vince_1} and the fact that both $\twoToInf{\bbU_k}$ and $\spectral{\hbD_k}$ are $\OP{N^{-1/2}}$(the latter being a consequence of Assumptions~\ref{assumption:full_Rank} and \ref{assumption:bounded_rv}, as shown in Lemma~\ref{lemma:sussman_Wk} of Appendix~\ref{app:eigenvalues}) then implies that $\twoToInf{\bbR_{k_3}}$ is $\OP{N^{-1}\log^{k\theta} N}$, as desired.

    Regarding $\bbR_{k_4} = \left(\bbW_k- \bbM_k \right)\bbU_k\left(\bbS_k\hbD_k^{-1/2} -\bbD_k^{-1/2}\bbS_k  \right)$ we have that:
    \begin{align*}
        \twoToInf{\bbR_{k_4}} &\leq \infnorm{\bbW_k-\bbM_k}\twoToInf{\bbU_k\left(\bbS_k\hbD_k^{-1/2} -\bbD_k^{-1/2}\bbS_k\right)}\\
        &\leq \sqrt{N}\spectral{\bbW_k-\bbM_k}\twoToInf{\bbU_k}\spectral{\bbS_k\hbD_k^{-1/2} -\bbD_k^{-1/2}\bbS_k},
    \end{align*}
    where in the second inequality we have used that $\infnorm{\bbA}\leq \sqrt{n}\spectral{\bbA}$ for all $\bbA \in \reals^{n\times m}$. Using Lemma \ref{lemma:A-P_spectral_norm} and Proposition \ref{prop:gallagher_1} together with the fact that $\twoToInf{\bbU_k}$ is $\OP{N^{-1/2}}$ then shows the desired bound for $\bbR_{k_4}$.

    This leaves us with the task of bounding the $2\to\infty$ norm of $\bbR_{k_2}$, which we remind the reader is defined as:
    \begin{align*}
    \bbR_{k_2} = \left(\bbI- \bbU_k\bbU_k^\top \right)\left(\bbW_k- \bbM_k \right)\left(\hbU_k -\bbU_k\bbS_k  \right)\hbD_k^{-1/2} .
    \end{align*}
    Therefore:
    \begin{align}
        \twoToInf{\bbR_{k_2}} \leq \twoToInf{\bbI - \bbU_k\bbU_k^\top } \twonorm{\bbW_k- \bbM_k} \twonorm{\hbU_k -\bbU_k\bbS_k } \twonorm{\hbD_k^{-1/2}}.
        \label{eq:R_2_split}
    \end{align}
    For the first term we have that:
    \begin{align*}
        \twoToInf{\bbI - \bbU_k\bbU_k^\top} \leq \twoToInf{\bbI} + \twoToInf{\bbU_k\bbU_k^\top} \leq 1 + \twoToInf{\bbU_k}\spectral{\bbU_k^\top}.
    \end{align*}
    Since $\spectral{\bbU_k^\top}=\spectral{\bbU_k} \leq \sqrt{N}\twoToInf{\bbU_k}$ we have that $\spectral{\bbU_k^\top} = \OP{1}$ because $\twoToInf{\bbU_k} = \OP{N^{-1/2}}$. Therefore
    \begin{align*}
        \twoToInf{\bbI - \bbU_k\bbU_k^\top} \leq 1+\OP{N^{-1/2}} = \OP{1}.
    \end{align*}
    Also note that:
    \begin{align*}
       \spectral{\hbU_k-\bbU_k\bbS_k}& \leq \, \spectral{\hbU_k-\bbU_k\bbU_k^\top\hbU_k} + \spectral{\bbU_k\left(\bbU_k^\top\hbU_k-\bbS_k\right)}\\
       & \leq \, \spectral{\hbU_k-\bbU_k\bbU_k^\top\hbU_k} + \spectral{\bbU_k}\spectral{\bbU_k^\top\hbU_k-\bbS_k}.
    \end{align*}
    By Proposition \ref{prop:vince_2} the first term is $\OP{N^{-1}\log^{k\theta}N}$, whereas the second term is $\OP{ N^{-1}\log^{k\theta} N}$ because $\spectral{\bbU_k} = \OP{1}$ and $\spectral{\bbU_k^\top\hbU_k-\bbS_k}$ is $\OP{N^{-1}\log^{k\theta} N}$ by virtue of Proposition \ref{prop:procrustes}. Therefore:
    \begin{align*}
    \spectral{\hbU_k-\bbU_k\bbS_k} = \OP{N^{-1}\log^{k\theta} N}.
    \end{align*}
    Combining this with Lemma \ref{lemma:A-P_spectral_norm} and the fact that $\spectral{\hbD_k^{-1/2}}$ is $\OP{N^{-1/2}}$, from~\eqref{eq:R_2_split} we conclude $\twoToInf{\bbR_{k_2}}=\OP{N^{-3/2}\log^{2k\theta} N}$, which concludes the proof.
\end{proof}

\section{Consequences of Assumptions \ref{assumption:full_Rank} and \ref{assumption:bounded_rv} regarding the largest eigenvalues of $\bbM_k$ and $\bbW_k$}\label{app:eigenvalues}

In this section, we show that Assumptions \ref{assumption:full_Rank} and \ref{assumption:bounded_rv} imply that for all $k\geq 0$, the nonzero eigenvalues of $\bbM_k := \bbX[k]\bbX^\top[k]$ are $\EquivP{N}$, and that the same holds for the top $d$ eigenvalues of $\bbW_k := \bbW^{(k)}$. We begin by proving this for $\bbM_k$. The proof scheme follows that of \cite[Proposition 4.3]{sussman2014consistent}. In what follows, the notation $\lambda_i(\bbA)$ denotes the $i$-th largest-magnitude eigenvalue of matrix $\bbA$.

\begin{lemma}\label{lemma:sussman_Mk}
    Let $(\bbW,\bbX_k) \sim \mathrm{WRDPG}(F)$, where $F$ satisfies Assumption \ref{assumption:full_Rank} and $\bbW$ satisfies Assumption \ref{assumption:bounded_rv}. Then for all $k\geq 0$ it holds that $\lambda_i(\bbM_k)=\EquivP{N}$ for $i=1,\dots,d$, while $\lambda_i(\bbM_k)=0$ for $i=d+1,\dots,N$.
\end{lemma}

\begin{proof}
    That $\lambda_i(\bbM_k)=0$ for $i=d+1,\dots,N$ is immediate since $\bbM_k=\bbX[k]\bbX^\top[k]$, so $\bbM_k$ is a $N\times N$ matrix with rank at most $d$. For $i=1,\dots, d$, note that
    \begin{align*}
\lambda_i(\bbM_k)=\lambda_i(\bbX[k]\bbX^\top[k])=\lambda_i(\bbX^\top[k]\bbX[k]).
    \end{align*}
    Now, each entry of $\bbX^\top[k]\bbX[k]$ is:
    \begin{align*}
        (\bbX^\top[k]\bbX[k])_{ij} = \sum_{l=1}^N \left(\bbx_{l}[k]\right)_i\left(\bbx_{l}[k]\right)_j,
    \end{align*}
    where $\left(\bbx_{l}[k]\right)_i$ denotes the $i$-th entry of vector $\bbx_{l}[k]$. Then $(\bbX^\top[k]\bbX[k])_{ij}$ is a sum of independent random variables, each with expectation $\E{\left(\bbx_{l}[k]\right)_i\left(\bbx_{l}[k]\right)_j}=(\bbDelta_k)_{ij}$, where $(\bbDelta_k)_{ij}$ is the $(i,j)$ entry of the second moment matrix $\bbDelta_k$ from Assumption \ref{assumption:full_Rank}.
    Since, for each $k$, the inner product between the rows of $\bbX[k]$ equals the $k$-th moment of some random variable--which, by the definition of our model, we assume to be finite for all $k$--it follows that all rows of $\bbX[k]$ (which lie in $\mathbb{R}^d$) must have bounded norm, with a bound that does not depend on $N$. Therefore, each term in the above sum is bounded by some constant $L_k > 0$, so by Hoeffding's inequality \cite[Theorem 2.6.2]{vershynin2018hdpbook} we have that:
    \begin{align*}
        \Pc{\left\lvert (\bbX^\top[k]\bbX[k])_{ij} - N(\bbDelta_k)_{ij} \right\rvert \geq t} \leq 2 \exp\left( \frac{-2t^2}{N^2L_k^2}\right).
    \end{align*}
    Choosing $t=\frac{L_k}{\sqrt{2}}N^{1/2}\log^{1/2}N$ then shows that
    \begin{align*}
    \left\lvert (\bbX^\top[k]\bbX[k])_{ij} - N(\bbDelta_k)_{ij} \right\rvert = \OP{N^{1/2}\log^{1/2}N}.
    \end{align*}
    Taking a union bound then shows that $\Fro{\bbX^\top[k]\bbX[k]-N \bbDelta_k}$ is also $\OP{N^{1/2}\log^{1/2}N}$, and since the spectral norm is dominated by the Frobenius norm, then $\spectral{\bbX^\top[k]\bbX[k]-N \bbDelta_k} = \OP{N^{1/2}\log^{1/2}N}$. A corollary of Weyl's inequality \cite[Corollary 7.3.5]{horn2012matrix} then implies that:
    \begin{align*}
        \left\lvert\lambda_i\left(\bbX^\top[k]\bbX[k]\right) - N \lambda_i(\bbDelta_k)\right\rvert = \OP{N^{1/2}\log^{1/2}N}.
    \end{align*}
    Since $\lambda_i(\bbDelta_k) = \EquivP{1}$, this in turn implies that $\lambda_i\left(\bbX^\top[k]\bbX[k]\right) = \EquivP{N}$, which completes the proof.
\end{proof}

The next lemma shows that the top $d$ eigenvalues of $\bbW_k$ are $\EquivP{N}$, while the remaining ones are within $N^{1/2}\log^{1/2}N$ of zero with high probability.

\begin{lemma}\label{lemma:sussman_Wk}
    Under the same hypotheses as Lemma \ref{lemma:sussman_Mk}, for all $k\geq 0$ it holds that $\lambda_i(\bbW_k)=\EquivP{N}$ for $i=1,\dots,d$, while $\lambda_i(\bbW_k)=\OP{\log^{k\theta}N}$ for $i=d+1,\dots,N$.
\end{lemma}
\begin{proof}
    Again, by \cite[Corollary 7.3.5]{horn2012matrix} we have that:
    \begin{align*}
        \left\lvert\lambda_i\left(\bbW_k\right) -  \lambda_i\left(\bbM_k\right) \right\rvert \leq \spectral{\bbW_k - \bbM_k}.
    \end{align*}
    Therefore, using Lemma \ref{lemma:A-P_spectral_norm} we have that $ \left\lvert\lambda_i\left(\bbW_k\right) -  \lambda_i\left(\bbM_k\right) \right\rvert = \OP{\log^{k\theta}N}$. Applying Lemma \ref{lemma:sussman_Mk} implies the desired result.
\end{proof}

\section{Chebyshev Polynomial Reformulation}
\label{app:chebyshev_recovery}

We provide a detailed description of the Chebyshev-based reformulation of the moment-recovery system in \eqref{eq:vandermonde}. By expressing the system on the basis of Chebyshev polynomials of the first kind, we markedly improve the problem's conditioning.  

\subsection{Transformation of monomials into Chebyshev basis}

Any monomial $x^k$ of degree $k\le K$ can be uniquely expanded in the Chebyshev basis $\{T_j(x) : j=0,1,\dots,K\}$:
\begin{align*}
  x^k = \sum_{j=0}^K c_{k,j}\,T_j(x),
\end{align*}
where the coefficients $c_{k,j}$ are computed using recurrence relations or discrete orthogonality integrals. In practice, these coefficients can be precomputed up to degree $K$ using the recurrence:
\begin{align*}
  T_0(x) &= 1,\\
  T_1(x) &= x,\\
  T_{j+1}(x) &= 2x\,T_j(x) - T_{j-1}(x), \quad j\ge1,
\end{align*}
and the identity relating monomials to Chebyshev polynomials (e.g., via trigonometric definitions or forward recurrence).  

\subsection{Construction of the Chebyshev–Vandermonde matrix}

Let $\{v_0, v_1, \dots, v_R\}$ denote the support points of the discrete distribution. We define the Chebyshev--Vandermonde matrix $\bbV_{\text{C}} \in \reals^{(K+1) \times (R+1)}$ with entries
\begin{align*}
  (\bbV_{\text{C}})_{k,r} = T_k\bigl(v_r^*\bigr), \quad k = 0,\dots,K, \; r = 0,\dots,R,
\end{align*}
where $v_r^*$ are the support points mapped into the canonical interval $[-1,1]$ via affine scaling if necessary.  

The moment-recovery system then becomes an equivalent system in the Chebyshev basis. Let \(\bbC \in \mathbb{R}^{(K+1)\times(K+1)}\) be the matrix with entries \(C_{kj}=c_{k,j}\), so that the Chebyshev-moment vector  
\[
\bbm_{\text{C}} = \bbC\bbm
\]
is the projection of the original moment vector \(\bbm = [m_{ij}[0], \dots, m_{ij}[K]]^\top\) onto the Chebyshev basis. The system in the Chebyshev basis then reads  
\[
\bbV_{\text{C}}\,\bbp = \bbm_{\text{C}}.
\] 

\subsection{Numerical solution and stability}

Because Chebyshev polynomials constitute an orthogonal basis on $[-1,1]$ with respect to the weight $(1-x^2)^{-1/2}$, the condition number of $\bbV_{\text{C}}$ grows only polynomially in $K$, rather than exponentially as in the monomial basis. Therefore, solving for $\bbp$ via a standard least-squares or regularized inversion
\begin{align*}
  \hat{\bbp} = \argmin_{\bbp \ge 0} \|\bbV_{\text{C}} \,\bbp - \bbm\|_2^2
\end{align*}
offers improved numerical stability. Additional regularization or non-negativity constraints may be imposed to further enhance stability and ensure valid probability estimates.  

\section{Solution to the Moments Problem by Entropy Maximization}\label{app:maximum_entropy}

In Section~\ref{sec:generation_cont}, we formulated the problem of identifying a pdf that maximizes differential entropy subject to a set of moment constraints. Here we detail how this constrained optimization problem can be effectively tackled using a primal-dual approach. Specifically, we derive the corresponding dual formulation and present a gradient-based method for solving it, which is both computationally efficient and well-suited to our setting.

The problem is to find a pdf $g_{ij}$ that maximizes the differential entropy
\begin{align*}
    S(g_{ij}) = -\int_{\Omega}g_{ij}(x)\log g_{ij}(x)\, dx
\end{align*}
subject to the moments constraints
\begin{align}
    \int_{\Omega} x^kg_{ij}(x)\, dx = \bbx_i^\top[k]\bbx_j[k],\,\, \text{ for all } k=0,1,\dots, K,
    \label{eq:moments_constraints}
\end{align}
where $ \Omega \subset \reals $ denotes the known support of the distribution. 

From an optimization viewpoint, we wish to solve the primal problem:
\begin{align}
    \max_{g \in \mathcal{F}} S(g) \,\, \text{ s. to } \int_{\Omega} x^kg(x)\, dx - m_k=0,\,\, \text{ for all } k=0,1,\dots, K,
    \label{eq:maximum_entropy_primal}
\end{align}
where $m_k=\bbx_i^\top[k]\bbx_j[k]$ and $\mathcal{F}$ is the space of all probability distributions that are supported in $\Omega$ (as before, dependence of $g$ and $m_k$ on $i$ and $j$ is omitted for readability).

To solve \eqref{eq:maximum_entropy_primal}, we introduce Lagrange multipliers $\lambda_k$ and the Lagrangian functional\footnote{Following \cite{kapur1992entropy} we have used $\lambda_0-1$ as the first Lagrange multiplier for convenience.}
\begin{align}
    L(g,\bblam) = S(g)-(\lambda_0-1)\left( \int_{\Omega} g(x)\, dx - m_0\right) - \sum_{k=1}^K \lambda_k \left( \int_{\Omega} x^kg(x)\, dx - m_k\right),
    \label{eq:lagrangian}
\end{align}
where $\bblam=[\lambda_0, \lambda_1, \dots,\lambda_K]^\top\in \reals^{K+1}$. We will find a solution to \eqref{eq:maximum_entropy_primal} by solving the dual problem:
\begin{align}
    \min_{\bblam\in \reals^{K+1}} d(\bblam),
    \label{eq:maximum_entropy_dual}
\end{align}
where $d:\reals^{R+1} \mapsto \reals$ is the dual function
\begin{align*}
    d(\bblam) = \max_{g\in\mathcal{F}} L(g,\bblam) .
\end{align*}
Using the Euler-Lagrange equation for the calculus of variations we conclude that the function $g_{\bblam}$ that extremizes $L$ is the one that makes the functional derivative $\displaystyle \frac{\delta L(g,\bblam)}{\delta g(x)}$ equal to zero. By differentiation, we see that such a condition amounts to
\begin{align}
    \frac{\delta L(g,\bblam)}{\delta g(x)}\Big|_{g_{\bblam}(x)} = 0 \Leftrightarrow g_{\bblam}(x) = \exp\left(-\sum_{k=0}^K \lambda_k x^k\right).
    \label{eq:lagrangian_maximum}
\end{align}
Looking at the second variation of $L(g,\bblam)$ we conclude that this $g_{\bblam}$ maximizes the Lagrangian, so the dual function can be computed as
\begin{align*}
    d(\bblam) = L(g_{\bblam},\bblam) = \sum_{k=0}^K \lambda_k m_k + \int_\Omega \exp \left( -\sum_{k=0}^K \lambda_k x^k \right) dx - m_0.
\end{align*}

One can check that $d(\bblam)$ is a convex function of $\lambda_0,\lambda_1,\dots,\lambda_K$ (see, e.g., \cite[Example 2.3]{kapur1992entropy}). Since its partial derivatives are equal to
\begin{align*}
\frac{\partial d(\bblam)}{\partial \lambda_k} = m_k- \int_\Omega x^kg_{\bblam}(x)\, dx
\end{align*}
we have that whenever the gradient of $d(\bblam)$ is zero, then $g_{\bblam}$ satisfies the moments constraints \eqref{eq:moments_constraints}. In other words, if $\bblam^*$ is a solution to the dual problem \eqref{eq:maximum_entropy_dual}, then the function $g_{\bblam^*}$ given by \eqref{eq:lagrangian_maximum} both maximizes the Lagrangian \eqref{eq:lagrangian} and satisfies the moments constraints \eqref{eq:moments_constraints}, i.e., it is a solution to the primal problem \eqref{eq:maximum_entropy_primal}. This implies that strong duality holds for the maximum entropy problem. Therefore, by solving the dual problem, we can recover the maximum entropy pdf we are seeking. Since the dual objective is convex and unconstrained, it can be solved using any standard convex optimization algorithm. In practice, we solve it using the BFGS algorithm \cite{nocedal2006numerical} available through SciPy’s \texttt{minimize} function.

\begin{remark}
If a function $g_{\bblam}$ such as that in  \eqref{eq:lagrangian_maximum} satisfies the moments constraints \eqref{eq:moments_constraints}, then it has maximum entropy among all pdfs $g$ supported in $\Omega$ that satisfy such constraints. Indeed, we have that,  since \mbox{$\log g_{\bblam}(x) = -\sum_{k=0}^K \lambda_k x^k$}, 
    \begin{align*}
        S(g_{\bblam}) &=- \int_\Omega g_{\bblam}(x) \log g_{\bblam}(x) \, dx =  \sum_{k=0}^K \lambda_k \int_\Omega g_{\bblam}(x) x^k \, dx =\sum_{k=0}^K \lambda_k \int_\Omega g(x) x^k \, dx \\
        &= -\int_\Omega g(x) \log g_{\bblam}(x)\, dx,
    \end{align*}
where in the third equality we have used that both $g_{\bblam}$ and $g$ satisfy the moments constraints \eqref{eq:moments_constraints}. Then
    \begin{align*}
        S(g_{\bblam})-S(g) = \int_\Omega g(x) \log \left(  \frac{g(x)}{g_{\bblam}(x)}\right) dx = D_{KL}(g||g_{\bblam}) \geq 0
    \end{align*}
    where $D_{KL}$ is the Kullback-Leibler divergence.
\end{remark}

\begin{example}
We showcase our method by testing its ability to recover an exponential distribution from its first four moments. Since the pdf of an exponential random variable with parameter $\alpha$ is $\alpha e^{-\alpha x}$, it can be written in the form of \eqref{eq:lagrangian_maximum} by setting $\lambda_0 = -\log \alpha$, $\lambda_1 = \alpha$ and $\lambda_k=0$ for all $k\geq 2$, so we can check whether the Lagrange multipliers $\lambda_k$ obtained with our gradient descent solution are close to these values. The results for an exponential random variable with parameter $\alpha=2$ are depicted in Figure \ref{fig:max_ent_cont_ex}, where we show the $\lambda_k$'s upon convergence of our GD algorithm for 100 random initializations of $\bblam$. We also show the Lagrange multipliers obtained with the method proposed in \cite{saad2019pymaxent}. In that work, the authors estimate $\bblam$ by using Newton's method to approximate the solutions of the system of nonlinear equations that arises from imposing that the function $g_{\bblam}$ in \eqref{eq:lagrangian_maximum} satisfies the moments constraints \eqref{eq:moments_constraints}, i.e., the system
\begin{gather*}
\int_{\Omega} x^{k}\exp\left( -\lambda_0 - \lambda_1 x - \lambda_2 x^2 s- \dots -\lambda_R x^K \right) dx = m_k \quad \forall \, k=0, \dots, K.
\end{gather*}

As shown in Figure \ref{fig:max_ent_cont_ex} the Lagrange multipliers that we obtain using this approach are, most of the time, quite far from the actual values. This is because Newton's method is sensitive to initialization and heavily relies on having an initial guess that lies on the solution's basin of attraction. Our method, in contrast, always converges to the true solution. That is because our dual function is convex, so as long as the maximum entropy problem has a solution our method will converge to it no matter where it starts.

\begin{figure}[t]
    \centering
    \includegraphics[width=\linewidth]{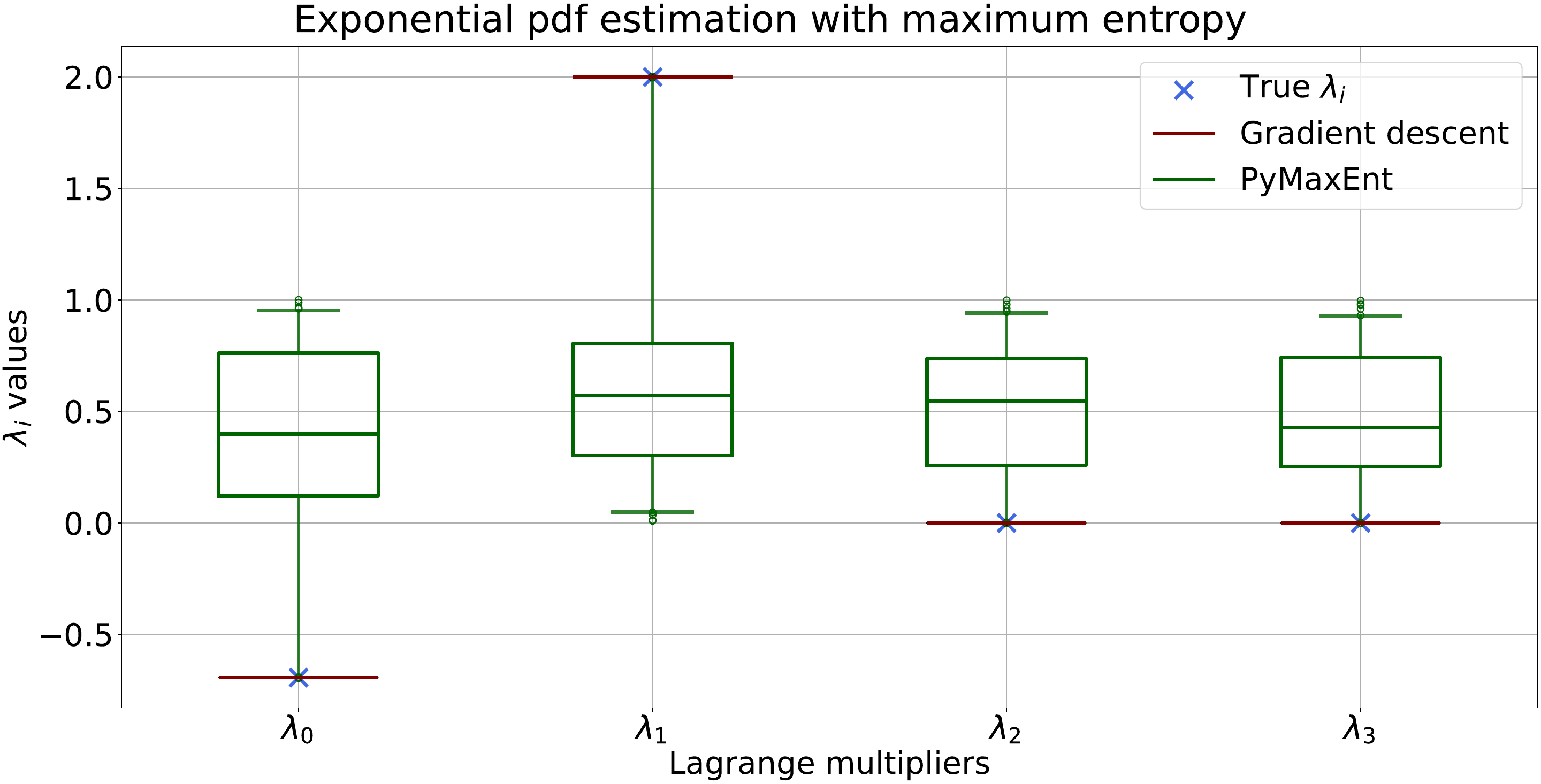}
    \caption{Box plots of Lagrange multipliers for maximum entropy estimation of an exponential rv distribution via gradient-descent (red) and the method from \cite{saad2019pymaxent} (PyMaxEnt, green) for 100 random initializations. Gradient descent always converges to the true value, while PyMaxEnt does not.}
    \label{fig:max_ent_cont_ex}
\end{figure}
\end{example}

\begin{remark}
This maximum entropy approach can be easily modified to accommodate discrete distributions. This is useful when we have access to fewer moments than the amount of symbols the discrete random variable takes, in which case we can no longer use the procedure described in Section \ref{sec:discrete_generation}. 
    
Assuming that our random variable takes on the values $v_0,\dots, v_R$, we wish to find its corresponding probabilities $p_0,\dots,p_R$ by maximizing Shannon's entropy
    \begin{displaymath}
        S(p_0,\dots,p_R) = -\sum_{r=0}^R p_r \log p_r
    \end{displaymath}
    subject to the moments constrains
    \begin{displaymath}
        \sum_{r=0}^R v_r^k p_r = m_k \,\,\, \forall k=0,\dots,K.
    \end{displaymath}
    We define the Lagrangian in analogy to \eqref{eq:lagrangian} and by differentiation we find that it has a maximum when $p_r = \exp\left( -\sum_{k=0}^K \lambda_k v_r^k\right)$. Thus, the dual function can be computed as:
    \begin{displaymath}
        d(\bblam) = \sum_{k=0}^K \lambda_k m_k + \sum_{r=0}^R \exp \left( -\sum_{k=0}^K \lambda_k v_r^k \right) - m_0,
    \end{displaymath}
    which, as before, is a convex function of $\lambda_0,\dots,\lambda_K$. Therefore, we can find its minimizer using any standard convex optimization algorithm and from it compute the $p_r$'s that maximize Shannon's entropy.
\end{remark}

\end{appendix}

\section*{Accompanying code}
The code to generate Figures 3 to 12 in this work is available at \url{https://github.com/bmarenco/wrdpg}.

\section*{Acknowledgments}
Part of the results in this paper were presented at the \textit{2024 Asilomar Conference on Signals, Systems, and Computers}~\cite{marenco2024wrdpg}.

The authors would like to thank Matías Valdés for helpful discussions on optimization related to entropy maximization.

\section*{Funding}
The first, second, third and fourth authors were partially supported by CSIC (I+D project 22520220100076UD) and ANII (SticAmsud LAGOON project).

The fifth author was supported in part by NSF award EECS-2231036.

\bibliographystyle{plainnat}  
\bibliography{citations}

\end{document}